\documentclass[onefignum,onetabnum]{siamonline181217}

\usepackage{amsfonts,amsmath}
\usepackage{graphicx}
\usepackage{epstopdf}
\usepackage{algorithmic}
\usepackage{amsmath,amssymb,enumitem}
\usepackage{color}
\usepackage[english]{babel}
\usepackage{mathrsfs,dsfont}
\usepackage{array}
\usepackage{mathrsfs}
\usepackage[mathscr]{euscript}

\ifpdf
  \DeclareGraphicsExtensions{.eps,.pdf,.png,.jpg}
\else
  \DeclareGraphicsExtensions{.eps}
\fi

\newsiamremark{remark}{Remark}
\newsiamremark{assumption}{Assumption}
\crefname{hypothesis}{Hypothesis}{Hypotheses}
\newsiamthm{claim}{Claim}
\newsiamremark{example}{Example}

\def\QED{~\rule[-1pt]{5pt}{5pt}\par\medskip}
\renewenvironment{proof}{{\bf Proof: \ }}{ \hfill \QED}

\let\VEC  \boldsymbol
\let\ALP  \mathcal
\let\FLD  \mathscr

\newcommand{\ebox}{\hfill $\Box$}

\newcommand{\transpose}{\mathsf{T}}
\newcommand{\ind}[1]{\mathds{1}_{\{#1\}}}
\newcommand{\beq}[1]{\begin{align} #1 \end{align}}
\newcommand{\beqq}[1]{\begin{align*} #1 \end{align*}}
\renewcommand{\Re}{\mathbb{R}}
\newcommand{\Na}{\mathbb{N}}
\newcommand{\Z}{{\mathbb{Z}_{+}}}
\newcommand{\ex}[1]{\mathds{E}\left[#1\right]}
\newcommand{\pr}[1]{\mathds{P}\left\{#1\right\}}

\renewcommand{\matrix}[2]{\left[\begin{array}{#1} #2 \end{array}\right]}

\headers{Markov Chains Induced by Recursive Stochastic Algorithms}{A. Gupta, G. Tendolkar, H. Chen, J. Pi}

\title{Some Limit Properties of Markov Chains Induced by Recursive Stochastic Algorithms\thanks{Submitted to the editors.
\funding{The authors gratefully acknowledge support from ARPA-E NEXTCAR Program, NSF ECCS Grant 1610615, and NSF CRII Award 1565487.}}}

\author{Abhishek Gupta, Hao Chen, Jianzong Pi, Gaurav Tendolkar\thanks{Abhishek Gupta, Hao Chen, and Jianzong Pi is with the Electrical and Computer Engineering Department, The Ohio State University, 2015 Neil Avenue, Columbus, OH 43210, USA.  (\email{gupta.706@osu.edu}, \email{chen.6945@osu.edu}, \email{pi.35@osu.edu}). Gaurav Tendolkar is with Microsoft Corp., 599 N Mathilda Avenue, Sunnyvale, CA 94085, USA. \email{gatendol@microsoft.com}.}
}

\usepackage{amsopn}

\usepackage{tikz}
\usetikzlibrary{shapes}

\usepackage{amsopn}

\begin{document}

\maketitle

\begin{abstract}
Recursive stochastic algorithms have gained significant attention in the recent past due to data driven applications. Examples include stochastic gradient descent for solving large-scale optimization problems and empirical dynamic programming algorithms for solving Markov decision problems. These recursive stochastic algorithms approximate certain contraction operators and can be viewed within the framework of iterated random operators. Accordingly, we consider iterated random operators over a Polish space that simulate iterated contraction operator over that Polish space. Assume that the iterated random operators are indexed by certain batch sizes such that as batch sizes grow to infinity, each realization of the random operator converges (in some sense) to the contraction operator it is simulating. We show that starting from the same initial condition, the distribution of the random sequence generated by the iterated random operators converges weakly to the trajectory generated by the contraction operator. We further show that under certain conditions, the time average of the random sequence converges to the spatial mean of the invariant distribution. We then apply these results to logistic regression, empirical value iteration, and empirical Q value iteration for finite state finite action MDPs to illustrate the general theory develop here.
\end{abstract}



\section{Introduction}
There has been a surge of interest in using randomization to reduce computational burden in machine learning and reinforcement learning. For instance, in training neural networks with a large amount of data, stochastic gradient descent is frequently employed instead of the usual gradient descent. In data-driven Markov decision problems, empirical dynamic programming and generative models have been employed to determine approximately optimal policies and value functions. In these algorithms, instead of computing the expected value of certain functions at each step of the iteration, one computes the empirical expected value that is rather easy to compute if enough data is available. This simple trick reduces the runtime to determine a reasonably good solution. 

It turns out that the outputs of these recursive stochastic algorithms (RSAs) can be viewed as Markov chains. Indeed, if the parameters of the algorithm do not change with iteration, then the RSAs can be thought of as an iterated random operator acting onto certain spaces. Consider, for instance, the case of stochastic gradient descent, where the stepsize remains constant, data samples picked at every iteration are i.i.d., and the number of data samples remain constant. Then, each step of the stochastic gradient descent algorithm can be viewed as a random operator. To see this, let us consider the problem of minimizing a sum of $N$ functions, $L_i:\Re^n\rightarrow\Re$, $i=1,\ldots,N$:
\beqq{\min_{x\in\Re^n} L(x) = \frac{1}{N}\sum_{i=1}^N L_i(x).}
In usual gradient descent, one fixes a stepsize $\beta>0$, and runs the iteration
\beqq{y_{k+1} = y_k - \beta \nabla_x L(y_k) =: T(y_k),}
where we used $T:\Re^n\rightarrow\Re^n$ to denote the exact gradient descent map. We note here that under reasonable assumptions on $L$ and $\beta$, $T$ becomes a contraction operator under some norm on $\Re^n$ (usually $\ell_2$ norm is used).

In contrast, in stochastic gradient descent, the operator applied at every step of the algorithm changes. At time step $k$ of the stochastic gradient descent algorithm, let $\ALP N_k:=\{i_1,\ldots,i_n\}$ be the set of $n$ indices that are sampled independently and uniformly from the set of all indices $\{1,\ldots,N\}$. Then, we have
\beqq{\hat z^n_{k+1} = \hat z^n_k - \frac{\beta}{n} \sum_{i\in\ALP N_k} \nabla_x L_i(\hat z^n_k) =:\hat T^n_k(\hat z^n_k).}
Since the set of (random) indices $\ALP N_k$ is i.i.d., the operator $\hat T^n_k$ is independent of the past maps and is ``identically distributed''. This implies that the (random) sequence $(\hat z^n_k)_{k\in\Na}$ is a Markov chain. It should also be noted that the exact gradient descent operator $T$ and stochastic gradient descent operator $\hat T^n_k$ are related -- $\hat T^n_k(x)$ is a consistent estimate of the $T(x)$ for any $x\in\Re^n$.  

A similar setup is considered in empirical dynamic programming using a generative model for dynamic decision process. Consider a controlled Markov process in which $s$ is the state of a system and $a$ is the action of the decision maker. Let $p(s'|s,a)$ denote the transition probability of the next state being $s'$ given the current state $s$ and action $a$ and $c(s,a)$ be the corresponding cost. We use $\alpha$ to denote the discount parameter. In the value iteration algorithm, one needs to compute $\ex{v(s')|s,a}$, where $v$ is some real-valued function of the state. This leads to the usual Bellman operator $T$ that acts on the space of value functions and is given by
\beqq{T(v)(s,a) = \min_{a}\bigg[ c(s,a)+\alpha\ex{v(s')|s,a}\bigg].}
It is not difficult to see that $T$ is a contraction operator with the contraction coefficient $\alpha$, when the space of value functions is endowed with the sup norm. If there is enough data, one can replace $\ex{v(s')|s,a}$ with its ``empirical'' average, given by $\frac{1}{n}\sum_{i=1}^n v(s'_i(s,a))$, where $\{s'_i(s,a)\}_{i=1}^n$ are $n$ samples of the next state given that the current state-action pair is $(s,a)$. Thus, the random Bellman operator $\hat T^n_k$ acts on the space of value functions and is given by
\beqq{\hat T^n_k(v)(s,a) = \min_{a}\left[ c(s,a)+\alpha\frac{1}{n}\sum_{i=1}^n v(s'_i(s,a))\right].}
In this case as well, $\hat T^n_k(v)$ is a consistent estimate of the $T(v)$ for any value function $v$. When the state or action space is continuous, then the above operator also features an additional function fitting component. The analysis of such algorithms involves understanding the error introduced due to function fitting, as well as the number of samples used at every iteration.

In fact, in both of the examples above, it is readily observed that the original operator $T$ is a contraction map. This is no accident -- contraction mapping theorem forms the bedrock of most of the proofs of convergence algorithms used in optimization or MDP problems. Some examples are noted below:
\begin{enumerate}
    \item The Bellman operator for a class of stochastic shortest path problems is a contraction operator under an appropriate weighted sup norm \cite{bertsekas1996neuro,bertsekas2011dynamic}.
    \item The Bellman operator for a class of average cost MDPs is a contraction operator under the span seminorm over the quotient space \cite{puterman2014,almudevar2014approximate}.
    \item Some variational inequality problems involve contraction operators under the usual 2 norm \cite[p. 143]{facchinei2007finite}.
    \item The Bellman operator for continuous state MDP is a contraction operator over the space of continuous and bounded functions over the state space of the MDP (it requires a variety of assumptions as elucidated in \cite{hernandez1996discrete,hernandez2012further,duflo2013random}).
    \item The resolvent of a strongly monotone operator is a contraction operator \cite{ryu2016primer}. Many other contraction operators used in the context of optimization algorithms are discussed in Section 5 of \cite{ryu2016primer}.
\end{enumerate}
In data driven applications, computing the exact (contraction) operator is either computationally expensive or impossible. Thus, one has to use random mappings, that are drawn independent of the past operators, and that simulate the effect of contraction mapping. The two examples explained above are merely instances of this more general methodology. The primary purpose of the paper is to devise sufficient conditions on the relationship between random maps $\hat T^n_k$ and the deterministic contraction operator $T$ so that (a) the convergence properties of new algorithms can be readily established, and (b) find common threads between the convergence.   

\subsection{Our Contribution}\label{sub:contri}
The primary contribution of this paper is to conceptually unify the convergence analysis of certain RSAs in optimization, 
machine learning, and reinforcement learning using the tools from the Markov chain theory. This is achieved by leveraging several results available for convergence and stability for Feller Markov chains established in \cite{breiman1960}, \cite[Sec. 18.5]{meyn2012}, \cite[Sec. 8]{borovkov1998}, and \cite{karr1975}. Our key contributions are summarized below.
\begin{enumerate}
\item Consider the deterministic sequence $y_k = T\circ \ldots \circ T(y_0)$ ($k$ compositions of the exact operator $T$) and the Markov chain $\hat z^n_k = \hat T^n_{k-1}\circ \ldots\circ \hat T^n_0(\hat z^n_0)$ with $\hat z^n_0 = y_0$. It is natural to assume that using the deterministic operator yields the best convergence guarantee. We use independent samples of data to approximate $T$ by $\hat T^n_k$, then it introduces error at every time step. How does this error accumulate? Most authors bound $\|\hat z^n_k - x^*\|$, where $x^*$ is the fixed point of the operator $T$. Instead, we are interested in the error $\|\hat z^n_k - y_k\|$, noting that triangle inequality yields 
\beq{\label{eqn:znkyk}\|\hat z^n_k - x^*\|\leq \|\hat z^n_k - y_k\|+\|y_k - x^*\| \leq \|\hat z^n_k - y_k\|+\alpha^k\|x_0-x^*\|.}

We derive a sufficient condition on the random operators $\hat T^n_k$ and its relationship with the exact operator $T$ so that the sequence of distributions over the sequence $(\hat z^n_k)_{k\in\Na}$ output by the RSA converges in weak* topology to the unit mass over the trajectory $(y_k)_{k\in\Na}$ output by the exact algorithm as the parameter $n\rightarrow\infty$. This further implies that $\hat z^n_k$ is close to $y_k$ with high probability. Using the inequality in \eqref{eqn:znkyk} above, we conclude that the error term $\|\hat z^n_k - x^*\|$ is less than  $\alpha^k\|x_0-x^*\|$ plus a small error with high probability.

We further show that the sufficient condition is satisfied in a sufficiently general class of problems encountered in stochastic gradient descent for strongly convex loss functions and synchronous empirical dynamic programming for MDPs with discounted cost criterion. These examples serve to illustrate how to apply the results to derive this property of an RSA.

\item Existence of invariant distributions of RSAs is an important property, as it implies some form of stability of the algorithm. A Markov chain does not admit invariant distribution if it features a cyclic behavior or diverges to infinity (there could be other reasons for the non-existence of an invariant distribution, but these two are more common). In an RSA, we do not usually expect a cyclic behavior due to randomization. Thus, if an RSA admits an invariant distribution, then it implies that the iterates will not diverge with high probability. Thus, establishing the existence of an invariant distribution is an important problem, which we address here.

We show that the Markov chains generated by many RSAs satisfy the weak Feller property, that is, if $f$ is a continuous and bounded function, then $\hat z^n_0\mapsto \ex{f(\hat z^n_1)|\hat z^n_0}$ is also a continuous and bounded function. The existence (and in some case, the uniqueness) of an invariant measure of Feller Markov chains has been presented in \cite{borovkov1998}. We apply these results to conclude that under some reasonable assumptions, the chains generated by stochastic gradient descent and empirical dynamic programming algorithms admit invariant distributions. In certain cases, we can show that this invariant distribution is unique.

\item There has been a sustained interest in using time-averages in stochastic gradient descent and deep Q learning. Particularly, references \cite{schmidt2017minimizing,merity2017regularizing,jain2018parallelizing,mandt2017stochastic,tripuraneni2018averaging} propose that fixing the stepsize in stochastic gradient descent algorithms and using the average of the tail of the random sequence leads to a better performance of the trained algorithm. Within the context of reinforcement learning, \cite{anschel2017averaged} and \cite{wai2018multi} propose averaging the deep Q function iterates to arrive at a solution with lower variance. 

Indeed, we show here that under some conditions on the random operators, the variance reduction property of time-average (or in these cases, tail average) is largely due to the fact that the Markov chain output by RSA may be admitting a unique invariant distribution. This part leverages the law of large numbers for Markov chains, presented in \cite{breiman1960,meyn2012}. 
\item We complement the theoretical contributions with numerical simulations of two RSAs -- stochastic gradient descent for logistic regression, empirical value iteration for discounted MDP with a generative model, and synchronous batch Q value iteration for discounted MDP.
\end{enumerate}

While we present complete proofs of two of our main results stated here, we admit that our proofs require minor tweaks of existing results in the literature. The need for presenting the complete proofs are twofold: Our hypotheses differ in some ways from the hypotheses presented in the standard texts, particularly in \cite{karr1975}, \cite[Section 18.5]{meyn2012}, and \cite{breiman1960}. Moreover, to construct the complete proofs using these texts under our hypotheses require substantial effort on the part of the reader. To ease this burden, we chose to furnish the complete proofs using the notation adopted here.  

\subsection{Previous Work}
Convergence proofs of randomized optimization and learning algorithms are usually obtained from specifically tailored arguments, which are not generalizable to other settings. For instance, the convergence of stochastic gradient descent, stochastic variance reduction gradient descent (SVRG), and stochastic average gradient (SAG) descent follow a completely different, and often involved, sequence of arguments \cite{bottou2010large,bonnabel2013stochastic,johnson2013accelerating,harikandeh2015stopwasting,schmidt2017minimizing,defazio2014saga}. The argument usually starts with identifying some conditions on the functions, such that for every iteration $k$, one can upper bound $L(\hat z^n_k) - L(x^*)$ (where we used the notation introduced above) by a function that decays as $k$ grows. These tailored methods usually also yield the convergence rates specific to those algorithms. 

It would be conceptually elegant to determine a set of more general conditions which can be readily applied to these algorithms and many of its variants to establish the asymptotic convergence to the fixed point of the map. The stochastic approximation theory is one such elegant theory \cite{borkar2000ode,huang2002ode,kushner2003stochastic,borkar2009stochastic}. There are two types of stochastic approximation algorithms -- one with decreasing (also called tapering or diminishing) stepsize and other one with constant stepsize. In decreasing stepsize algorithms, the stepsize has to converge to 0 as the number of iteration goes to infinity (the stepsize is not summable, but is square summable). This leads to the almost sure convergence guarantee to the fixed point in the limit. For constant stepsize, the theory says that the iterates will eventually enter a neighborhood of the fixed point and do a random walk within that neighborhood.

We now present a sample of decreasing stepsize RSAs whose convergence is ascertained using the stochastic approximation theory. Under reasonable assumptions on the loss function, stochastic gradient descent and distributed asynchronous gradient descent methods converge almost surely to the optimal solution \cite{tsitsiklis1986distributed,gelfand1991recursive,gelfand1993metropolis}. The convergence of reinforcement learning algorithms usually invoke some version of the stochastic approximation theorem. Reference \cite{jaakkola1994convergence,tsitsiklis1994asynchronous,abounadi2002stochastic,bertsekas2012q} studies the convergence of various types of Q learning algorithms developed for discounted cost or average cost MDPs with finite state and action spaces. The convergence of on-policy reinforcement learning algorithm SARSA is established in \cite{singh2000convergence}. More recently, the stochastic approximation theory has been used to establish the convergence of policy gradient, temporal difference, and other related methods in 
\cite{suttle2019multi,zhang2019global,zhang2019policy,zhang2020approximate,xiong2020non,xu2020non}. For more information on various reinforcement learning algorithms and their convergence proofs, we refer the reader to books  \cite{almudevar2014approximate,sutton2018reinforcement,bertsekas1996neuro} and recent survey papers \cite{zhang2019decentralized,zhang2019multi}.

Decreasing stepsize RSAs do not yield approximate solutions in a reasonable time frame. As a result, constant stepsize algorithms are gaining traction as a way to speed up the computation at the cost of tolerating a small error in the final result; see, for example, \cite{fort1999asymptotic,nemirovski2009robust,bach2014adaptivity,dieuleveut2017harder,dieuleveut2017bridging}, where constant stepsizes are used in the context of the stochastic gradient descent-type algorithms and \cite{beck2012error,srikant2019finite,o2010residential,schoknecht2003convergent} for their usage in the reinforcement learning algorithms. 

Constant stepsize stochastic approximation over finite dimensional state space has been studied in \cite{borkar2000ode,huang2002ode}, where the authors derive the asymptotic concentration results. It is well-understood that in stochastic gradient descent with constant stepsize, the sequence generated by the algorithm gets closer to the optimal solution, but then does a random walk around the optimal solution \cite{dieuleveut2017bridging}. The closeness of the random walk to the optimal solution depends on the number of random samples one uses at each iteration of the algorithm. Similarly, in Q learning algorithm, constant stepsize Q learning has been studied in \cite{beck2012error} (both synchronous and asynchronous version are studied). Convergence of constant stepsize temporal difference methods with linear function approximation is studied in \cite{lakshminarayanan2018linear,dalal2018td0,bhandari2018finite,gupta2019finite,srikant2019finite}. Empirical value iteration and their variants, studied in a number of works under varying assumptions
\cite{whitt1978approximations,whitt1979approximations,cooper2003convergence,munos2008finite,cooper2012performance,jain2010simulation,kalathil2014empirical,haskell2016,sidford2018near}, are also examples of constant stepsize stochastic approximation algorithms.

When the stepsize is taken as a constant in a RSA, then the output of the RSA forms a Markov chain. The goal of this paper is to study the limit properties of such a Markov chain. There are two ways the limit can be taken. Either the sample size $n$ used at every time step can grow to infinity or the number of iterations $k$ can escape to infinity. We study both the limits in this paper. We note here that the generality of the model and minimal assumptions do not allow us to derive a finite time guarantee, which has significant importance in the machine learning and the reinforcement learning communities. Further, our proof approach is not algorithm-dependent. We leave these important problems for a future work.

Our work is largely motivated by the analysis of empirical dynamic programming in \cite{haskell2016}. This work viewed empirical dynamic programs within the framework of iterated random operators. It used stochastic dominance based arguments to derive bounds on the asymptotic probability of error (between the random outputs of the algorithm and the optimal solution) being large. Inspired by this work, we extended the arguments to empirical relative value iteration in \cite{gupta2015edp}. We further relaxed some conditions on random operators assumed in \cite{haskell2016} in our follow up work \cite{gupta2018probabilistic}. The aim of this paper is to further expand the analysis and present conditions on random operators and its relationship to the exact operator to arrive at insightful conclusions about the random sequences generated by these RSAs.

\subsection{Outline of the Paper}
The paper is organized as follows. Section \ref{sec:problem} presents a common mathematical framework to study the problem of convergence and stability of Markov chains induced by RSAs. We also state the three main questions we address. Section \ref{sec:example} presents some motivating examples where the mathematical framework we develop can be applied. Through these examples, we also illustrate certain desirable properties that the random operators enjoy. In Section \ref{sec:weakconv}, we show that the distributions over the trajectories generated by RSA converges to the Dirac mass over the trajectory generated by the exact algorithm. This constitutes the first main result of the paper. In Section \ref{sec:invariant}, we study some sufficient conditions on the operators $\hat T^n_k$ such that the resulting Markov chain admits an invariant distribution. We also study conditions under which the invariant distribution is unique. Section \ref{sec:lln} then introduces the assumptions and establishes the weak law of large numbers for Markov chains. This constitutes the second main result of the paper. The proofs of the two main results are presented in Sections \ref{sec:mainproof} and \ref{sec:llnproof}. We finally conclude our discussion in Section \ref{sec:conclusion}.  

\subsection{Notations}
Let $(\ALP A,\rho)$ be a Polish space, which is defined as a complete separable metric space with metric $\rho$. We use $\wp(\ALP A)$ to denote the set of all probability distributions over $\ALP A$. We use $\ind{a}\in\wp(\ALP A)$ to denote the Dirac mass over $a\in\ALP A$. The notations $C_b(\ALP A)$ and $U_b(\ALP A)$ denote, respectively, the set of all continuous and bounded functions and uniformly continuous and bounded functions over the set $\ALP A$. We use $C(\ALP A)$ to denote the set of (possibly unbounded) continuous functions over the set $\ALP A$. We say that a sequence of measures $\{\mu_n\}_{n\in\Na}\subset\wp(\ALP A)$ converges to $\mu$ in weak* sense iff for every $f\in C_b(\ALP A)$, $\int fd\mu_n \rightarrow\int fd\mu$ as $n\rightarrow\infty$. This is usually also referred to as weak convergence in probability theory literature. 

We use $1_{\{\cdot\}}$ to denote an indicator function, which takes the value of 1 if $\{\cdot\}$ is true and 0 otherwise. By a slight abuse of notation, we also use $1_{F}$ to be the indicator function over a measurable set $F\subset\ALP A$.

\section{Problem Formulation}\label{sec:problem}
Let $(\ALP X,\rho)$ be a Polish space with metric $\rho$. Consider a contraction operator $T:\ALP X\rightarrow \ALP X$ with contraction coefficient $\alpha\in(0,1)$ and the unique fixed point denoted by $x^*$. This means
\beqq{\rho(T(x_1),T(x_2)) \leq \alpha\rho(x_1,x_2) \text{ and } T(x^*) = x^*.}
Starting from any initial point $y_0\in\ALP X$, define the iterates 
\beq{y_k = T(y_{k-1})\quad \text{ for }k\in\Na.\label{eqn:yk}}
By the Banach contraction mapping theorem, this iteration converges to $x^*$. In fact, we have
\beqq{\rho(y_k,x^*) \leq \alpha^k\rho(y_0,x^*).}
As discussed previously, in many instances, it is beneficial or required in many iterative algorithms to use randomization to evaluate an approximation of $T(x)$ using a random operator. We now formulate a framework to analyze the output of this RSA rigorously. 

Let $(\Omega,\FLD F,\mathbb{P})$ be a standard probability space, where $\Omega$ is the set of uncertainties, $\ALP F$ is the Borel $\sigma$-algebra over $\Omega$ and $\mathbb{P}$ be the probability distribution function over $\Omega$. Let $\hat T^n_k:\ALP X\times\Omega\rightarrow\ALP X$ be a random operator that is used at the $k^{th}$ iteration and is indexed by a natural number $n$. The index $n$ would capture, for instance, the stepsize, the number of random samples used to approximate the operator $T$, etc. Although $\hat T^n_k$ is a function of $\omega\in\Omega$, we will suppress this dependence for ease of exposition. Thus, $\hat T^n_k (x):=\hat T^n_k(x;\omega)$. We make the following assumption on the independence of the sequence of operators $\hat T^n_k$.

\begin{assumption}\label{assm:hatt}
For every $x\in\ALP X$, $\hat T^n_k(x)$ and $\hat T^n_{k'}(x)$ are statistically independent and identically distributed for $k\neq k'$.
\end{assumption}

\subsection{Key Questions}
Consider the stochastic process that starts from $\hat z^n_0 = y_0$ and define $\hat z^n_k = \hat T^n_{k-1}(\hat z^n_{k-1})$ for all $k\in\Na$. Due to Assumption \ref{assm:hatt} and the fact that $n$ does not change with time, the stochastic process $\{\hat z^n_k\}_{k\in\Na}$ is a time-homogeneous Markov chain. One can view $\hat z^n_k$ as an $\ALP X$-valued Markov chain with the Markov transition kernel given by
\beqq{Q_n(B|x) = \pr{\hat z^n_k \in B|\hat z^n_{k-1} = x}}
for any Borel set $B\subset\ALP X$. Note that $Q_n$ does not depend on the time index $k$, since what we have here is a time-homogeneous Markov process.

\subsubsection{Convergence of Distribution of Trajectories}
We are interested in deriving conditions on the random maps $\hat T^n_k$ under which the random sequence generated by RSA is close to the deterministic sequence generated by exact algorithm with high probability. Let us formulate the precise mathematical problem. We let $\mu_n\in\wp(\ALP X^\Na)$ denote the joint distribution of the sequence $(\hat z^n_0,\hat z^n_1,\hat z^n_2,\ldots)$. Endow $\ALP X^\Na$ with the product topology so that it becomes a Polish space. Then, $\mu_n$ is defined by
\beq{\label{eqn:mun}\mu_n(B_0\times B_1\times B_2\times \ldots) = \int_{B_0\times B_1\times B_2\times\ldots} \ind{y_0}(dx_0)Q_n(dx_1|x_0)Q_n(dx_2|x_1)\ldots,}
where $B_0,B_1,B_2,\ldots$ are Borel sets in $\ALP X$.

In the similar vein, one can also view the iterates $(y_k)_{k\in\Na}$ defined in \eqref{eqn:yk} as a Markov chain on the same probability space $(\Omega,\FLD F,\mathbb{P})$, with the distribution over this sequence defined by
\beq{\label{eqn:psi}\psi(B_0\times B_1\times B_2\times \ldots) = \int_{B_0\times B_1\times B_2\times\ldots} \ind{y_0}(dx_0)\ind{y_1}(dx_1)\ind{y_2}(dx_2)\ldots.}
This is a Dirac mass on the sequence $(y_0,y_1,\ldots)$. Our first result, stated in Section \ref{sec:weakconv}, proves that under a mild assumption on the random operators $\hat T^n_k$, the sequence of measures $\mu_n$ converges in the weak* sense to $\psi$.

A similar setup was considered by Karr \cite{karr1975}. It studies the convergence properties of a class of Feller Markov chains parametrized by $n$ such that the transition probability $Q_n$ converges in some sense to a transition probability $Q$ as $n\rightarrow\infty$. Although our assumptions are slightly different, the proof essentially imitates the one in \cite{karr1975} except for a couple of key steps. We also discuss numerical implication of this result in Section \ref{sec:weakconv}.

\subsubsection{Existence of Invariant Measures for Fixed \texorpdfstring{$n$}{}}
For the Markov chain $(\hat z^n_k)_{k\in\Na}$, one of the key questions is the existence of an invariant distribution. An invariant probability distribution of the Markov chain $(\hat z^n_k)_{k\in\Na}$ is a probability measure $\pi^n$ such that for any Borel set $B\subset\ALP X$, we have
\beqq{\pi^n(B) = \int_{\ALP X} Q_n(B|x)\pi^n(dx).}
A desirable property of a Markov chain is to have an invariant distribution, since it implies that the RSA satisfies a form of stability. More importantly, it implies that the Markov chain will not escape to infinity with probability 1. There is a large body of literature that studies the problem of the existence of invariant measures for Harris recurrent Markov chains that take values in continuous state spaces \cite{borovkov1998,meyn2012}. However, the Markov chain generated by the RSAs seldom satisfy the strong recurrence structure required for Harris recurrent chains. 

Instead, these chains satisfy the weak Feller conditions, for which there are limited results in the literature. Nonetheless, we show that many RSAs satisfy certain desirable properties, which can be leveraged to not only guarantee the existence of an invariant distribution, but also establish the uniqueness of the invariant distribution. These properties of the random operators are discussed in Section \ref{sec:invariant}. This further leads to strong conclusions about the weak law of large numbers, as we discuss next.

\subsubsection{Convergence of Time Average of Iterates}
The weak law of large numbers for independent and identically distributed (i.i.d.) random variables states that the time average of i.i.d. random variables converges to the mean of the distribution in probability under fairly mild conditions. In fact, such a version of the weak law of large numbers is also available for Feller Markov chains. This is established for Feller chains in \cite{breiman1960} for chains residing in a compact Hausdorff space with a unique invariant measure, and in \cite[Section 18.5]{meyn2012} for the non-compact case under certain technical conditions, which include existence of an invariant measure. It turns out that this result can be proved simply under the uniqueness of the invariant measure if the starting point $\hat z^n_0$ is chosen according to certain specific distribution (in fact, we do not need other technical conditions of \cite[Section 18.5]{meyn2012}). We prove this result in Section \ref{sec:lln}, the proof of which is adapted from the results from \cite{breiman1960} and \cite{meyn2012}.

We illustrate the theoretical results using numerical simulations for batch gradient descent, empirical value iteration, and synchronous  batch Q learning in Section \ref{sec:simulation}. 

\section{Motivating Examples}\label{sec:example}
We present here two examples where we illustrate how the random operator framework can be applied.

\subsection{Stochastic Gradient Descent in Logistic Regression}\label{sub:logistic}
Logistic regression has been widely used in many binary or multi-class classification problems. For simplicity, we consider the logistic regression with binary classification. Let $\mathcal{U} \subset\Re^m$ be the set of feature vectors. Let $(u^i,l^i)_{i=1}^N\subset\ALP U\times\{0,1\}$ denote the labeled dataset with $N$ data points and their labels. Our task is to model conditional probability distribution of label $l$ given the feature vector $u \in \mathcal{U}$. In logistic regression, we model $\pr{l=1\vert\ u^i}$ as $f(u^i;x)$ where $x\in \ALP X:=\Re^m$ are the parameters of $f$ to be learned from the data, where $f$ is defined below:

\begin{align*}
    f(u; x) = \sigma(u^\transpose x), \text{ where } \sigma(t) = \frac{1}{1+e^{-t}}.
\end{align*}

Our goal is to compute the parameter $x$ that maximizes the log likelihood (or equivalently, minimizes the negative log likelihood) given the labeled data. The log likelihood $L: \mathcal{X}\longrightarrow \Re$ of i.i.d data under conditional distribution $f$ is given by
\begin{align*}
    L(x) =  \frac{1}{N}\sum_{i=1}^N L_i(x), \text{ where } L_i(x):= -l^i\log f(u^i;x) - (1-l^i)\log(1-f(u^i;x)),
\end{align*}
It can be shown that the derivatives of $L_i$ are given by
\beqq{\nabla_x L_i(x) = (f(u^i;x)-l^i)u^i,\quad \nabla^2_{xx} L_i(x) = \Big(f(u^i;x)(1-f(u^i;x))\Big)u^iu^{i\transpose}.}
Consequently, each $L_i$ is a convex function, and thus, $L$ is a convex function over the space $\ALP X$. If the matrix $[u^1|\ldots|u^N]$ is full rank and $N>m$, then it immediately follows that $\nabla^2_{xx} L(x)$ is a full rank matrix with positive eigenvalues. Consequently, $L$ is strongly convex, and therefore has a unique minimum $x^*$. This minimum can be computed using the usual gradient descent algorithm. The algorithm starts at $x_0$, picked arbitrarily, and proceeds in the direction of $-\nabla_x L(x_k)$ in small steps of size $\beta$:
\begin{equation}\label{eq:lrexupdate}
    x_{k+1} = T(x_k) := x_k - \beta\nabla_x L(x_k) = x_k - \frac{\beta}{N}\sum_{i=1}^N ( f(u^i;x_k)-l^i)u^i,
\end{equation}
where $T:\ALP X\rightarrow\ALP X$ is the gradient descent map (dependent on the parameter $\beta$). It can be further shown that if $\beta$ is sufficiently small, then the operator $T$ is a contraction on $\mathcal{X}$, endowed with the Euclidean norm.  We note here that the above arguments would be true if $\{L_i\}_{i=1}^N$ is a collection of strongly convex and smooth loss functions. 

In practice, the exact gradient computation of loss function $L$ is computationally expensive as it requires evaluating $N$ gradients at every time step. Therefore, to ease the computational burden, a mini-batch Stochastic Gradient Descent (SGD) is employed. In the mini-batch SGD, at every step $k$, the gradient is approximated by a small, randomly sampled, subset (of size $n$) of the data set. To introduce this algorithm, let $\ALP N_k\subset\{1,\ldots,N\}$ be the randomly sampled subset of size $n$. The state is updated as
\begin{equation}\label{eq:minisgdalgo}
\hat z^n_{k+1} = \hat T^n_k(\hat z^n_k) = \hat z^n_k - \beta\frac{1}{n} \sum_{j=1}^n \nabla_{x}(L_j(\hat z^n_k)).
\end{equation}
Note that $\hat T^n_k$ is now a random operator. 
\begin{remark}\label{rem:hatt}
Any realization of this random operator with small values of $n$ need not be a contraction since $ \nabla_{x} L_j$ is a rank 1 positive semidefinite matrix. One could add a regularizer to the loss function to make it strongly convex. In particular, if the loss function is chosen as  
\beqq{L_i(x):= -l^i\log f(u^i;x) - (1-l^i)\log(1-f(u^i;x)) +\frac{\lambda}{2}\|x\|^2_2 \quad \text{ with } \lambda>0,}
then $\hat T^n_k$ is a contraction operator for a sufficiently small $\beta$ irrespective of the $n$ used. \ebox 
\end{remark}
Some obvious properties of these random operators are:
\begin{enumerate}
 \item $\hat T^n_k$ is continuous in $\hat z^n_k$.
 \item For every $\epsilon>0$ and $x\in\ALP X$, we have
 \beqq{\lim_{n\rightarrow\infty}\pr{\|\hat T^n_k(x)-T(x)\|_2>\epsilon} = 0.}
 \item Suppose that for any compact set $\ALP K\subset\ALP X$, $\|\nabla_x L_i(x)\|_2$ is uniformly bounded, that is, there exists $M_{\ALP K}\in\Re$ such that $\sup_{x\in \ALP K}\sup_{1\leq i\leq N } \|\nabla_x L_i(x)\|_2\leq M_{\ALP K}$. Then, for any $\epsilon>0$,
  \beq{\label{eqn:logeps}\lim_{n\rightarrow\infty}\sup_{x\in\ALP K}\pr{\|\hat T^n_k(x)-T(x)\|_2>\epsilon} = 0.}
  This statement can readily be proved using the Hoeffding inequality and the union bound.
\end{enumerate}
Let us depart from the specific case of logistic regression, and consider the case where $L_i$ can be any strictly concave function for all $i\in\{1,\ldots,N\}$. Then,  $\hat T^n_k$ satisfies the following property.
\begin{enumerate}
 \item[4.] If all the eigenvalues of the Hessian of $L_i(x)$ satisfy $0<m \leq \lambda(\nabla^2_{xx}L_i(x))\leq M<\infty$ for all $i\in\{1,\ldots,N\}$, then every realization of the random operator $\hat T^n_k$ is a contraction map with contraction coefficient $\hat \alpha^n_k\equiv \max\{\vert 1-\beta M \vert, \vert 1-\beta m \vert\} < 1$ for an appropriately picked $\beta>0$. 
\end{enumerate}

We now introduce the empirical dynamic programming algorithm in the context of value iteration for MDP with discounted cost criteria.

\subsection{Empirical Value Iteration for Discounted Cost MDP}\label{sub:edpdiscount}
Consider a Markov Decision Problem (MDP) problem described by 4 tuple $(\mathcal{S}, \mathcal{A}, c, p)$, where $\mathcal{S}$ is the finite state space, $\mathcal{A}$ is the finite action space, and $c: \mathcal{S}\times\mathcal{A}\longrightarrow \mathcal{R}$ is the cost function. The state transitions according to $s_{t+1} \sim p(\cdot|s_t,a_t)$. Let $\Gamma$ denote the set consisting of all possible deterministic policies $\gamma: \mathcal{S} \longrightarrow \mathcal{A}$. The infinite horizon discounted cost $v^{\gamma}: \mathcal{S} \longrightarrow \mathcal{R}$ starting from state $s$ and following policy $\gamma$ is given by
\begin{align*}
v^\gamma(s) := \ex{\sum_{k=0}^{\infty}\alpha^kc(s_k,a_k) \ \middle|\; s_0=s, a_k = \gamma(s_k)},
\end{align*}
where $\alpha\in(0,1)$ is the discount factor. The goal is to compute the optimal value $v^*(s) = \inf_{\gamma\in\Gamma} v^\gamma(s)$. Let $\mathcal{V}$ be the set of all $v: \mathcal{S} \longrightarrow\Re$; this space is isomorphic to the Euclidean space $\Re^{|\ALP S|}$.

It can be shown that the optimal infinite horizon discounted cost is a fixed point of a contraction map $T: \mathcal{V} \rightarrow \mathcal{V}$, where $T$ is the Bellman operator given by
\begin{equation}\label{eq:bellmanop}
   T(v)(s) = \inf_{a \in A}\Big\{ c(s, a) + \alpha \sum_{s'\in\ALP S} v(s') p(s'|s,a)\Big\}.
\end{equation}
Due to the Banach contraction mapping theorem, $T:\ALP V\rightarrow\ALP V$ admits a unique fixed point, which is equal to $v^*$. The iterative process of finding this unique fixed point is called the Value Iteration algorithm:
\begin{equation}\label{eq:valueiter}
\begin{aligned}
    &\text{Initialize $v_0$ randomly and iterate } v_{k+1}(s) = T(v_k)(s).
\end{aligned}
\end{equation}
In data driven applications, it is often the case that for all possible state-action pairs, multiple realizations of the next states are available. In this situation, we can replace the expectation in \eqref{eq:bellmanop} to the empirical average. This algorithm is referred to as empirical dynamic programming, and is written as
\begin{equation}\label{eq:apxbellmanop}
    \hat{T}^n_k(\hat v^n_k)(s) = \inf_{a \in \mathcal{A}}\left\{ c(s, a) + \alpha \frac{1}{n}\sum_{i=1}^n\hat v^n_k(s_{k,i}'(s,a))\right\},
\end{equation}
where $s_{k,i}'(s,a)$ are $n$ independent and identically distributed samples of the next state given the current state-action pair $(s.a)$, redrawn at every $k$ independently from the past realizations. The above intuition can be turned into an algorithm to determine an approximately optimal value function, and is known as the empirical value iteration algorithm:
\begin{equation}\label{eq:apxvalueiter}
\begin{aligned}
    &\text{Initialize $v_0$ randomly and let } \hat v^n_{k+1}(s) = \hat{T}^n_k(\hat v^n_k)(s).
\end{aligned}
\end{equation}
Note that $\hat{T}^n_k$ is a random operator, and its realization is dependent on the samples generated $(s_{k,i}'(s,a))_{i=1}^n$. The following properties of $\hat T^n_k$ are obvious:
\begin{enumerate}
 \item $\hat T^n_k$ is a contraction with contraction coefficient $\alpha$. Therefore, $\hat T^n_k$ is continuous.
 \item For every $\epsilon>0$ and $v\in\ALP V$, we have
 \beqq{\lim_{n\rightarrow\infty}\pr{\|\hat T^n_k(v)-T(v)\|_\infty>\epsilon} = 0.}
 \item In fact, we have a stronger property here. For any compact set $\ALP K\subset\ALP V$, we have for every $\epsilon>0$
 \beqq{\sup_{v\in\ALP K}\pr{\|\hat T^n_k(v)-T(v)\|_\infty>\epsilon} \leq 2|\ALP S||\ALP A|\exp\left( -\frac{\epsilon n}{|\ALP S|k^2}\right),}
 where $k = \max_{v\in\ALP K} \|v\|_\infty$. This immediately yields
 \beq{\label{eqn:evieps}\lim_{n\rightarrow\infty}\sup_{v\in\ALP K}\pr{\|\hat T^n_k(v)-T(v)\|_\infty>\epsilon} = 0.}
  \item Let $\ALP V$ be endowed with the partial order $\preceq$ such that $v_1\preceq v_2$ implies $v_1(s)\leq v_2(s)$ for all $s\in\ALP S$. If $c\geq 0$, then $\hat T^n_k$ satisfies 
 \begin{enumerate}
  \item If $v_0 = 0$, then $v_0\preceq \hat T^n_k(v_0)$.
  \item If $v_1\preceq v_2$, then $\hat T^n_k(v_1)\preceq \hat T^n_k(v_2)$.
  \item If $v_l\uparrow v$, then $\hat T^n_k(v_l)\rightarrow \hat T^n_k(v)$ as $l\rightarrow\infty$.
 \end{enumerate}
\end{enumerate}

\subsection{Observations}
Through the two examples above, we observed that the approximate operator $\hat T^n_k$ corresponding to the contraction operator $T$ is context dependent. In the case of stochastic gradient descent, it is constructed by picking certain loss functions randomly and then averaging their gradients. In the case of empirical dynamic programming, the approximate operator involves computing the empirical average of the future expected value. Nonetheless, there are some fundamental properties that the empirical operator satisfies in both situations. For instance, the property stated in \eqref{eqn:logeps} in the context of logistic regression is (mathematically) the same as the property stated in \eqref{eqn:evieps} in the context of empirical value iteration. Similarly, every realization of the random operator $\hat T^n_k$ is a contraction map under certain reasonable assumptions. We will consider more examples in Section \ref{sec:simulation}, where we show that these properties (or some minor variant of these properties) are enjoyed by other empirical dynamic programming algorithms as well.

The other important observation is that every realization of random operator $\hat T^n_k$ may also satisfy some other desirable properties. For instance, in the empirical value iteration example, if we endow $\ALP V$ with a partial order $\preceq$ and the cost is nonnegative, then every realization of $\hat T^n_k$ satisfies certain monotonicity properties. This is very useful in establishing the existence of unique invariant measure, as we show in Section \ref{sec:invariant}. This property is, unfortunately, not satisfied by the logistic regression problem. This property is also not satisfied by the empirical relative value iteration for the average cost MDP. However, we will show that the realizations of the random operators in these cases have some other desirable properties that lead to the existence and uniqueness of the invariant measure.

We now turn our attention to introducing our first main result in the next section.

\section{Weak* Convergence of the Distribution of Trajectories}\label{sec:weakconv}
We now study the convergence property of the sequence of distributions $\mu_n$, which is defined in \eqref{eqn:mun}. Before we study that, we need to ensure that the random operator $\hat T^n_k$ is ``close to'' the operator $T$ in some sense. Accordingly, we make the following assumption. 

\begin{assumption}\label{assm:suppr}
 For every compact set $K\subset\ALP X$, $\epsilon>0$, and  $\delta>0$, there exists $N_{\epsilon,\delta}(K)>0$ such that
\beqq{\sup_{x\in K}\pr{\rho(\hat T^n_k(x),T(x)) >\epsilon}<\delta \text{ for all } n\geq N_{\epsilon,\delta}(K).}
\end{assumption}

We recall here that this assumption is satisfied by the logistic regression and empirical value iteration for discounted cost MDP (see  \eqref{eqn:logeps} and \eqref{eqn:evieps} within the discussion at the end of Subsections \ref{sub:logistic} and \ref{sub:edpdiscount}). We are now in a position to introduce our first main result.

\begin{theorem}\label{thm:main}
If Assumptions \ref{assm:hatt} and \ref{assm:suppr} hold, then $\mu_n$ converges in weak* topology to $\psi$ as $n\rightarrow\infty$, where $\psi$ is defined in \eqref{eqn:psi}.
\end{theorem}
\begin{proof}
The proof is based on the proof by \cite{karr1975}, except that our hypotheses are slightly different from those in \cite{karr1975}. For completeness, we present a proof in Section \ref{sec:mainproof}.
\end{proof}

Levy-Prohorov's metric over the space of probability measures over Polish spaces metrizes the weak* topology \cite{parthasarathy2005}. Generally speaking, if distributions of two random variables are close to each other in Levy-Prohorov's metric, then it does not imply that the random variables will be close to each other. As an instance, the Levy-Prohorov's metric between the measures of two independent and identically distributed random variables is zero, but the difference between the random variables themselves is not zero. If one of the random variables is deterministic (that is, its distribution is a Dirac mass), then the random variable must be close to the deterministic variable with high probability. This is established in the lemma below.

\begin{figure}[bth]
\centering\newcommand{\Tube}[6][]%
{   \colorlet{InColor}{#4}
    \colorlet{OutColor}{#5}
    \foreach \I in {1,...,#3}
    {   \pgfmathsetlengthmacro{\h}{(\I-1)/#3*#2}
        \pgfmathsetlengthmacro{\r}{sqrt(pow(#2,2)-pow(\h,2))}
        \pgfmathsetmacro{\c}{(\I-0.5)/#3*100}
        \draw[InColor!\c!OutColor, line width=\r, #1] #6;
    }
}

\begin{tikzpicture}
    \draw[->] (0,0) -- (8,0);
    \Tube{9mm}{100}{pink!25}{red!75}
        {(0,4) to[out=-85,in=179] (7.5,0.5)};
    \draw[thick] (0,4.08) to [out=-85,in=179] (7.55,0.5);
    \fill[fill] (0,4) circle [radius=0.05];
    \node[above right] at (0,4) {$y_0 = \hat z^{n}_0$};
    \draw[fill] (1,2) circle [radius=0.05];
    \node[below] at (1,1.9) {$y_1$};
    \draw[fill] (2,1.35) circle [radius=0.05];
    \node[above right] at (2,1.3) {$y_2$};
    \draw[fill] (3,0.99) circle [radius=0.05];
    \node[below] at (3,0.85) {$y_3$};
    \draw[fill] (4,0.76) circle [radius=0.05];
    \node[below] at (4,0.65) {$y_4$};
    \draw[fill] (5,0.63) circle [radius=0.05];
    \node[above right] at (5,0.65) {$y_5$};
    \draw[fill] (6,0.55) circle [radius=0.05];
    \node[below] at (6,0.5) {$y_6$};
    \draw[fill] (7,0.51) circle [radius=0.05];
    \node[below] at (7,0.45) {$y_7$};
    
    \draw[dashed, green!70!black, thick] (0,4) -- (1,2.4);
    \draw[dashed, green!70!black, thick] (1,2.4) -- (2,0.8);
    \draw[dashed, green!70!black, thick] (2,0.8) -- (3,1.5);
    \draw[dashed, green!70!black, thick] (3,1.5) -- (4, 1);
    \draw[dashed, green!70!black, thick] (4,1) -- (5,0.4);
    \draw[dashed, green!70!black, thick] (5,0.4) -- (6,1);
    \draw[dashed, green!70!black, thick] (6,1) -- (7,0.7);
    \draw[<->,red!70!black, thick] (0.53,2.5) -- node[pos = 0.25, below] {$\epsilon$} (0.2,2.2);
    
    \fill[fill] (1,2.4) circle [radius=0.05];
    \node[above right] at (1,2.4) {$\hat z^{n}_1$};
    \fill[fill] (2,0.8) circle [radius=0.05];
    \node[below left] at (2,0.8) {$\hat z^{n}_2$};
    \fill[fill] (3,1.5) circle [radius=0.05];
    \node[above right] at (3,1.5) {$\hat z^{n}_3$};
    \fill[fill] (4,1) circle [radius=0.05];
    \node[above right] at (4,1) {$\hat z^{n}_4$};
    \fill[fill] (5,0.4) circle [radius=0.05];
    \node[below left] at (5,0.5) {$\hat z^{n}_5$};
    \fill[fill] (6,1) circle [radius=0.05];
    \node[above right] at (6,1) {$\hat z^{n}_6$};
    \fill[fill] (7,0.7) circle [radius=0.05];
    \node[above right] at (7,0.7) {$\hat z^{n}_7$};

    \draw(0,-.1) node[below]{$0$} -- (0,0.1);
    \draw(1,-.1) node[below]{$1$} -- (1,0.1);
    \draw(2,-.1) node[below]{$2$} -- (2,0.1);
    \draw(3,-.1) node[below]{$3$} -- (3,0.1);
    \draw(4,-.1) node[below]{$4$} -- (4,0.1);
    \draw(5,-.1) node[below]{$5$} -- (5,0.1);
    \draw(6,-.1) node[below]{$6$} -- (6,0.1);
    \draw(7,-.1) node[below]{$7$} -- (7,0.1);
    \node[below] at (8,-0.1) {$k$};
    
\end{tikzpicture}
\caption{\label{fig:tube}Illustration of the behavior of $\hat z^n_k$ with varying values of $k$. Under the hypotheses of Theorem \ref{thm:main}, for every $\epsilon>0$, there exists $N_\epsilon>0$ such that $d_P(\mu_n,\psi)<\epsilon$ for every $n\geq N_\epsilon$. For most values of $k$, $\hat z^n_k$ stays within $\epsilon$ ball around $y_k$. In the illustration above, $\hat z^n_3$ is not within $\epsilon$ ball around $y_3$.}
\end{figure}
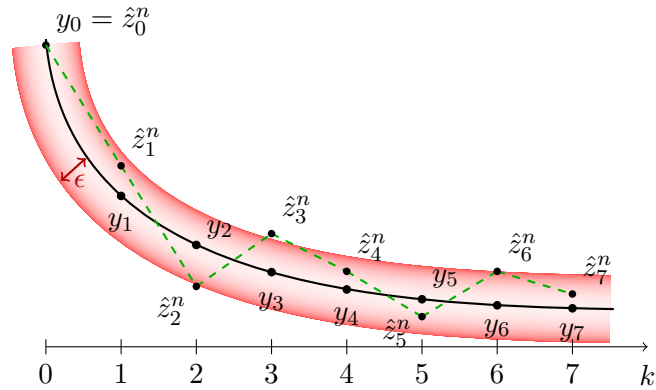

Let $\ALP A$ be a Polish space with metric $\rho_{\ALP A}$. Let $d_P$ be the Levy-Prohorov's metric on the space of probability measures $\wp(\ALP A)$ over $\ALP A$. This metric is defined as follows. For a Borel set $A\subset\ALP A$, let $A^\epsilon$ be defined as
 \beqq{A^\epsilon = \{a\in\ALP A: \rho_{\ALP A}(a,b) <\epsilon, b\in  A\}. }
 Let $\mu,\nu\in\wp(\ALP A)$. Then, $d_P(\mu,\nu)$ is defined by
\beqq{d_P(\mu,\nu) = \inf\Big\{\epsilon>0: \mu(A) < \nu(A^\epsilon)+\epsilon, \nu(A)<\mu(A^\epsilon)+\epsilon \text{ for all Borel sets } A\subset\ALP A\Big\}.}
We are now in a position to introduce our next result. We believe that this result may not be new, but we could not locate a reference where this result is proved.

\begin{lemma}\label{lem:aastar}
Let $\mu\in\wp(\ALP A)$ and $\ind{a^*}$ be a unit mass at point $a^*\in\ALP A$. If $d_P(\mu,\ind{a^*})<\epsilon$, then for any random variable $W$ distributed according to the law $\mu$, we have
 \beqq{\pr{\rho_{\ALP A}(W,a^*)\geq \epsilon}<\epsilon.}
\end{lemma}
\begin{proof}
 Let $B_\epsilon$ be an open $\epsilon$ ball around $a^*$. Let $(B_\epsilon^\complement)^\epsilon$ be defined as
 \beqq{(B_\epsilon^\complement)^\epsilon = \{a\in\ALP A: \rho_{\ALP A}(a,b) <\epsilon, b\in  B_\epsilon^\complement\}. }
 Note that $(B_\epsilon^\complement)^\epsilon = \ALP A\setminus a^*$, which implies $\ind{a^*}((B_\epsilon^\complement)^\epsilon) = 0$. Let $W$ be a random variable distributed according to the law $\mu$.  Then, from the definition of Levy-Prohorov's metric, we know that
 \beqq{\mu(B_\epsilon^\complement) < \ind{a^*}((B_\epsilon^\complement)^\epsilon)+\epsilon = \epsilon.}
 The proof then follows from noting that
 \beqq{\pr{\rho_{\ALP A}(W,a^*)\geq \epsilon} =\mu(B_\epsilon^\complement).}
 The proof is established.
\end{proof}

As a consequence of the lemma above, we conclude that since the distribution $\mu_n$ converges to the Dirac delta function $\psi$, it implies that for $n$ sufficiently large, the random sequence generated by RSA lies within a small tube around the trajectory induced by the deterministic contraction operator with high probability. This is illustrated in the Figure \ref{fig:tube}. 

To see this, let $\hat{\VEC{z}}^n:= (\hat z^n_0,\hat z^n_1,\ldots)$ and $\VEC y = (y_0, y_1,\ldots)$. Endow the space $\ALP X^{\Na}$ with the following metric:
\beqq{\rho_{\ALP X^{\Na}}(\hat{\VEC{z}}^n, \VEC y) := \sum_{k=0}^\infty \frac{1}{2^k}\rho(\hat z^n_k,y_k).}
It can be readily established that $\rho_{\ALP X^{\Na}}$ defined above is a metric on $\ALP X^\Na$. Then, due to Lemma \ref{lem:aastar}, we conclude that
\beqq{\pr{\rho_{\ALP X^{\Na}}(\hat{\VEC{z}}^n, \VEC y) <\epsilon}\geq 1-\epsilon.}
Next, note that if $\hat{\VEC{z}}^n$ satisfies $\rho_{\ALP X^{\Na}}(\hat{\VEC{z}}^n, \VEC y) <\epsilon$, then $\hat z^n_k$ is within $\epsilon$ neighborhood of $y_k$ for most of the $k\in\Na$.

\section{Existence of Invariant Measures for Fixed \texorpdfstring{$n$}{}}\label{sec:invariant}
In this section, we identify conditions on the operators $(\hat T^n_k)_{k\in\Na}$ so that the Markov chain $\hat z^n_k$ admits a unique invariant distribution, denoted by $\pi^n$, as $k\rightarrow\infty$. The existence of an invariant measure yields insight about ``stability'' of an RSA. In particular, if there is no invariant distribution, then it is likely the case that the sequence generated by RSA can blow up with positive probability. Thus, by tweaking the RSA (for instance, by changing the stepsize or increasing the number of samples), one can ensure that the sufficient conditions noted below are satisfied, thereby establishing that the RSA is stable and yields finite values with probability 1. 

In the case where there exists unique invariant measure under certain assumptions on the initial condition, then it means that any element in the tail of the random sequence generated by the RSA will have its law as the invariant distribution. This is a crucial step in proving that the time average of $f(\hat z^n_k)$ for any $f\in C_b(\ALP X)$ converges in probability to the spatial average $\int fd\pi^n$ of the function with respect to the invariant measure $\pi^n$. This important result is established in the next section.

To state the assumptions, we drop the subscript $k$ wherever possible since the statistical properties of $\hat T^n_k$ and $\hat T^n_{k'}$ are independent and identical to each-other as long as $k\neq k'$. Below, we list three assumptions under which we can show that $\hat z^n_k$ admits an invariant distribution.

To introduce our first assumption, we need to assume a partial order $\preceq$ exists on the normed space $\ALP X$ (we will drop the completeness requirement on the space $\ALP X$ for this assumption). For instance, $\ALP X$ could be the Euclidean space with the natural partial order, wherein $x,y\in\Re^n$, $x\preceq y$ implies $x_i\leq y_i$ for all $i\in\{1,\ldots,n\}$. The limit of an increasing sequence $(x_k)_{k\in\Na}$ satisfying $x_1\preceq x_2 \preceq x_3 \preceq \ldots$ under this partial order may escape to infinity. Thus, it is required to bound the space $\ALP X$ to ensure that the sequences do not diverge along any of the coordinates. Another example of a normed space with partial order is the space of measurable functions from a Euclidean space $\Re^n$ to $[-M,M]$, denoted by $\ALP L_\infty(\Re^n,[-M,M])$ (where $M$ is a fixed positive real number). This space features a natural partial order, wherein $f_1\preceq f_2$ iff $f_1(s)\leq f_2(s)$ for every $s\in\Re^n$.

\begin{assumption}\label{assm:monotone}
The following conditions are satisfied:
\begin{enumerate}[leftmargin=1.5em]
\item $\ALP X$ is a bounded normed space (not necessarily Polish) with partial order $\preceq$. This ordering satisfies the following property: For any sequence $(x_k)_{k\in\Na}$ satisfying $x_1\preceq x_2 \preceq x_3 \preceq \ldots$, there exists a minimal element $\bar x\in\ALP X$ such that $x_k\preceq \bar x$ for all $k\in\Na$. This is denoted as $x_k\uparrow \bar x$.
\item The operator $\hat T^n$ satisfy
\begin{enumerate}
 \item Monotonicity 1: exists $x_0\in\ALP X$ such that $x_0\preceq \hat T^n(x_0)$ almost surely.
 \item Monotonicity 2: If $x_1\preceq x_2$, then $\hat T^n(x_1)\preceq \hat T^n(x_2)$ almost surely.
 \item Continuity: If $x_k\uparrow x$, then $\hat T^n(x_k)\rightarrow \hat T^n(x)$ as $k\rightarrow\infty$ almost surely.
\end{enumerate}
\end{enumerate}
\end{assumption}

Assumption \ref{assm:monotone} is satisfied in Markov decision processes with non-negative cost functions.  This has been noted in Subsection \ref{sub:edpdiscount} in the context of empirical value iteration for discounted cost criterion. However, this is also satisfied in MDP for total cost criterion with an absorbing state and having a proper policy (that is, there is a stationary policy under which the MDP terminates in an absorbing state with probability 1). The proof of the last claim follows from arguments similar to the one made in Subsection \ref{sub:edpdiscount}; see, for example, Chapter 1 of \cite{bertsekas1996neuro}, and standard texts  on MDPs with total cost criterion \cite{hernandez1996discrete,hernandez2012further,bertsekas2013abstract}. 

For RSAs involving MDPs, the point $x_0$ can be chosen easily. Assuming that the costs are non-negative, one could take the zero function as the initial value function or Q function -- and the Monotonicity 1 condition is satisfied almost surely.  

\begin{assumption}\label{assm:compact}
For $n\in\Na$ and $m\in\Na$, let $\hat \alpha_m$ denote the Lipschitz coefficient of $\hat T^n_m \circ \hat T^n_{m-1}\circ \ldots \circ \hat T^n_1$. The following conditions are satisfied:
\begin{enumerate}[leftmargin=1.5em]
\item $\ALP X$ is a compact Polish space.
\item For $n\in\Na$, there exists $m\in\Na$, such that for any $\epsilon>0$, there exists a $\delta_\epsilon>0$ such that
\beqq{\hat \alpha_m \leq 1,\quad \pr{\hat \alpha_m>1-\epsilon}<\delta_\epsilon. }
\end{enumerate}
\end{assumption}

The notable point in Assumption \ref{assm:compact}  is that the assumption requires $\ALP X$ to be compact. It is satisfied in empirical value iteration for MDP with average cost criterion, as long as we project the value functions outside a sufficiently large compact set back to that compact set. We adopted this approach earlier in \cite{gupta2015edp} to ensure that the value functions obtained through repeated use of empirical operators do not blow up. During simulations, however, we never needed to use projection, as the value functions were bounded. 

\begin{assumption}\label{assm:contraction}
The following conditions are satisfied:
\begin{enumerate}[leftmargin=1.5em]
\item $\ALP X$ is a Polish space.
\item There exists $a,b>0$ such that the operator $\hat T^n$ satisfy
\beqq{\pr{\rho\Big(\hat T^n(x^*),x^*\Big)>\epsilon} = \pr{\rho\Big(\hat T^n(x^*),T(x^*)\Big)>\epsilon}\leq \frac{a}{\epsilon^b},}
where $x^*$ is the fixed point of $T$.
\item Let $\hat \alpha^n$ denote the Lipschitz coefficient of $(\hat T^n)_{k\in\Na}$. Then,
\beqq{\ex{\hat \alpha^n }<\infty,\quad \ex{\log(\hat \alpha^n)} < 0.}
\end{enumerate}
\end{assumption}

Assumption \ref{assm:contraction} is satisfied in stochastic gradient descent of strongly convex and smooth functions as noted in Subsection \ref{sub:logistic}. This is also trivially satisfied in the empirical value iteration for an MDP with a discounted reward criterion, as we have noted in Subsection \ref{sub:edpdiscount}. 

\begin{remark}
 Assumption \ref{assm:contraction}(ii) is typically proved using concentration of measures results, like Hoeffding inequality \cite{hoeffding1963,ledoux2001concentration,raginsky2013concentration} or the theory of empirical processes \cite{pollard1984convergence,van1996weak}. Some references where such inequalities have been used in empirical dynamic programming for continuous state MDPs are \cite{munos2008finite,haskell2019universal,gupta2018probabilistic}.
\end{remark}

Our next theorem summarizes the main result of this section. Let $\mu^n_k$ denote the distribution of $\hat z^n_k$.

\begin{theorem}\label{thm:invariant}
Suppose that Assumption \ref{assm:hatt} holds. Additionally, if either one of three assumptions, Assumptions \ref{assm:monotone}, \ref{assm:compact}, \ref{assm:contraction}, holds, then there exists an invariant measure $\pi^n$ such that $\mu^n_k$ converges to $\pi^n$ in weak* topology. Further, the invariant measure is unique under either of the following circumstances:
\begin{enumerate}[leftmargin=1.5em]
\item Assumption \ref{assm:monotone} holds, and the RSA is always initialized with $x_0$, where $x_0$ is defined in Assumption \ref{assm:monotone}. 
\item Assumption \ref{assm:compact} holds.
\item Assumption \ref{assm:contraction} holds.
\end{enumerate}
\end{theorem}
\begin{proof}
Under Assumption \ref{assm:monotone}, the existence of invariant measures is proved in \cite[Theorem 8.1, p. 79-81]{borovkov1998}. Under Assumption \ref{assm:compact}, the existence of invariant measures is proved in \cite{bhattacharya1989} and \cite[Theorem 8.2, p. 82-83]{borovkov1998}. Under Assumption \ref{assm:contraction}, the existence and uniqueness result is established in \cite[Theorem 1.1, p. 87]{diaconis1999}.
\end{proof}

\begin{remark}
We can replace the assumption of $\ALP X$ being a compact Polish space in Assumption \ref{assm:compact} by making the following assumption. There exists $x_0\in\ALP X$ such that if $\hat z^n_0 = x_0$, then for any $\delta>0$, there exists $N_\delta\in\Na$ such that for any $k\geq 1$, we have
\beq{\label{eqn:compactprime}\pr{\rho\Big(\hat z^n_0,\hat z^n_k\Big)>N_\delta}<\delta.}
If the above condition and Assumption \ref{assm:compact} (2) holds, then one can show that for any initial condition $\hat z^n_0\in\ALP X$, a unique invariant distribution exists. For a proof, we refer the reader to \cite[p. 179]{borovkov1998}. However, proving \eqref{eqn:compactprime} is satisfied in usual RSAs appears to be difficult in our experience. \hfill $\Box$
\end{remark}

We now turn our attention to establishing the law of large numbers for time averages of the outputs from a RSA.

\section{Averaging of Iterates and The Weak Law of Large Numbers}\label{sec:lln}
RSAs with constant stepsize and averaging of iterates have received significant attention recently. Nemirovski et. al. in \cite{nemirovski2009robust} develops an algorithm for stochastic gradient descent with averaging of the last few iterates and shows that the algorithm is robust to stepsize selection and yields better results when properties of the loss functions (such as strong convexity parameter, Lipschitz coefficient of the derivative, etc.) are unknown. Along similar lines, \cite{bach2013non,bach2014adaptivity,dieuleveut2017harder} study stochastic gradient descent based algorithms with constant stepsize and averaging and derive the finite time guarantees on the loss achieved with averaged output. We note here that the constant stepsize algorithms with averaging are fundamentally different from Polyak-Ruppert averaging used in stochastic approximation \cite{polyak1992acceleration,kushner2003stochastic}, where the stepsizes decreases at a rate slower than $1/k$.

Motivated by the above references, in this section, we consider the  problem of convergence of the sequence of averages of the random sequence $(\hat z^n_k)_{k\in\Na}$ (or of $(f(\hat z^n_k))_{k\in\Na}$ for some $f\in C_b(\ALP X)$). As stated in Subsection \ref{sub:contri}, averaging of the last few outputs of an RSA is used to reduce the variance in the final output of the algorithm. Specifically, if we terminate the iteration of RSA at $k$, and output $\hat z^n_k$, there is a small chance that this output is far from $x^*$. It is generally believed that if we average the last few outputs (assuming none of them are $\infty$), then it reduces the variance in the output. We establish this result rigorously here, under the assumption that the Markov chain has a unique invariant distribution.

We will assume throughout that the Markov chain $(\hat z^n_k)_{k\in\Na}$ admits an invariant distribution (the precise assumption is stated in the sequel). The time-average of the Markov chain is precisely the law of large numbers for Markov chains. It has been studied within the context of Markov chains over compact spaces in \cite{breiman1960} and over general locally compact spaces in \cite[Sec 18.5]{meyn2012}. Let us formulate the problem precisely.

Consider a continuous function $f\in C_b(\ALP X)$. Let $\pi^n$ denote the invariant measure of the Markov chain $(\hat z^n_k)_{k\in\Na}$. We have already studied the conditions under which such an invariant measure would exist in Theorem \ref{thm:invariant}. In what follows, we show that under relatively mild assumptions, we have
\beq{\label{eqn:asconv}\frac{1}{K}\sum_{k=0}^{K-1} f(\hat z^n_k) \rightarrow \int f(x) \pi^n(dx) \text{ in probability as } K\rightarrow\infty.}
Note that in the expression above, the term $\int f(x) \pi^n(dx)$ is the spatial average of the function $f$ at the invariant distribution. Thus, what we show is that the time average converges in probability to the spatial average -- which implies that $(f(\hat z^n_k))_{k\in\Na}$ is a discrete time ergodic process.

Let us define the operator $F:C_b(\ALP X)\rightarrow C_b(\ALP X)$ and its adjoint $F^*:\wp(\ALP X)\rightarrow\wp(\ALP X)$, where $\wp(\ALP X)$ is endowed with the weak* topology, as follows:
\beqq{F(f)(x) = \int_{\ALP X} f(y) Q_n(dy|x) = \ex{f(\hat T^n(x))}, \quad F^*(\mu)(dy) = \int_{\ALP X} Q_n(dy|x) \mu(dx).}

\begin{assumption}\label{assm:unique}
The distribution $\mu$ of the initial condition $\hat z^n_0$ is picked from a set $\ALP M\subset\wp(\ALP X)$. Either of the following two conditions holds:
\begin{enumerate}[leftmargin=1.5em]
 \item There exists a unique invariant measure $\pi^n$ such that for any $\mu\in\ALP M$, $(F^*)^k(\mu)$ converges in weak* topology to  $\pi^n$.
 \item There exists a unique invariant measure $\pi^n$ such that for any $\mu\in\ALP M$, the averaged operator satisfies
 \beqq{\lim_{k\rightarrow\infty} \frac{1}{k}\sum_{i=0}^{k-1} (F^*)^k(\mu) = \pi^n,}
 where $(F^*)^k(\mu)$ denotes $k$ compositions of $F^*$ applied on $\mu$ and the convergence is in weak* sense.
\end{enumerate}
\end{assumption}

From the Stolz-Cesaro theorem \cite[Theorem 2.7.1, p.59]{choudary2014real}, it is easy to show that if Assumption \ref{assm:unique}(1) holds, then Assumption \ref{assm:unique}(2) holds as well. However, the converse may not be true. To prove our next result, we only need Assumption \ref{assm:unique}(2) to hold. It should be noted that Assumption \ref{assm:unique}(1) may be rather easy to prove using Theorem \ref{thm:invariant}.

\begin{theorem}\label{thm:lln}
If Assumptions \ref{assm:hatt} and \ref{assm:unique} hold, then \eqref{eqn:asconv} holds. 
\end{theorem}
\begin{proof}
The proof essentially follows the steps in \cite{breiman1960} and \cite[Theorem 18.5.1, p. 478]{meyn2012}, except that we relax the assumption on compactness of the state space as assumed in \cite{breiman1960} and replace the hypotheses in \cite{meyn2012} with Assumption \ref{assm:unique}. For completeness, a detailed proof is presented in Section \ref{sec:llnproof}.
\end{proof}

One way the presentation of Assumption \ref{assm:unique} departs from the traditional Markov chain literature is as follows. It is generally assumed that $\ALP M = \wp(\ALP X)$, that is, for every $\mu\in\wp(\ALP X)$, $(F^*)^k(\mu)$ converges in weak* topology to  $\pi^n$. This is a very strong assumption from the applicability viewpoint in RSAs. In particular, it is possible to pick the most suitable initialization for the RSAs, which implies that $\ALP M$ can be picked appropriately. For example, in empirical dynamic programming for MDPs, one can initialize the value function to be 0. Then, we can utilize Assumption \ref{assm:monotone} (with $x_0 = 0$) to establish the existence of a unique invariant distribution using Theorem \ref{thm:invariant}. Incidentally, for the law of large numbers to hold, we do not need the stronger condition of $\ALP M = \wp(\ALP X)$.

We now have a corollary of the result above, which we capture as a theorem below due to its importance and applicability to RSAs.

\begin{theorem}\label{thm:acompact}
Suppose that $\ALP X \subset \Re^n$ is a compact set. Suppose further that either of the following conditions hold:
\begin{enumerate}[leftmargin=1.5em]
\item Assumption \ref{assm:monotone} holds, and the RSA is always initialized with $x_0$, where $x_0$ is defined in Assumption \ref{assm:monotone}. 
\item Assumption \ref{assm:compact} holds.
\item Assumption \ref{assm:contraction} holds.
\end{enumerate}
Consider the average $\hat a^n_k$ of the Markov chain
\beqq{\hat a^n_k = \frac{1}{k}\sum_{t=0}^{k-1} \hat z^n_t.}
If Assumptions \ref{assm:hatt} holds, then for any initialization $\hat z^n_0\in\ALP X$, a unique invariant distribution $\pi^n$ exists and $\hat a^n_k$ converges almost surely to the mean $\bar a^n$ of the distribution $\pi^n$. 
\end{theorem}
\begin{proof}
The existence of a unique invariant measure is due to Theorem \ref{thm:invariant}. Since $\ALP X$ is a compact set, we can take $f(x) = x_i$ to conclude that $(\hat a^n_k)_i$ converges almost surely to $\bar a^n_i$ by \cite{breiman1960}.  
\end{proof}

The result in \eqref{eqn:asconv} requires $f$ to be a bounded function over $\ALP X$. We now consider the case where $f$ is a potentially unbounded continuous function. To establish essentially the same result, we need to make the following additional assumption.
\begin{assumption}\label{assm:bounded}
There exists a random variable $C(\omega)$ such that
\beqq{\sup_{k\geq 0} f(\hat z^n_k) \leq C(\omega) \text{ almost surely},}
and $C(\omega)<\infty$ for $\mathbb{P}$-almost all $\omega\in\Omega$.
\end{assumption}
\begin{theorem}\label{thm:bounded}
Suppose that $f:\ALP X\rightarrow\Re$ is a continuous function (potentially unbounded). If Assumptions \ref{assm:hatt}, \ref{assm:unique}, and \ref{assm:bounded} hold, then \eqref{eqn:asconv} holds.    
\end{theorem}
\begin{proof}
 The proof is presented in Subsection \ref{sub:bounded}.
\end{proof}
It is easy to see in practice if Assumption \ref{assm:bounded} holds or not, particularly when $\ALP X = \Re^n$. If the random iterates are uniformly bounded during a single run of the RSA, then Assumption \ref{assm:bounded} holds along that trajectory. If almost all independent runs of the RSA is not expected to ``blow up'', then the time average of the iterates is likely going to have the variance reduction property--the time average converges in probability to the spatial average.

We now turn our attention to illustrating the application of Theorems \ref{thm:main}, \ref{thm:lln}, and \ref{thm:bounded} in various optimization and empirical dynamic programming algorithms.

\section{Numerical Simulations}\label{sec:simulation}
In this section, we complement the theoretical results proved above with extensive numerical simulations. We conduct simulations of minibatch stochastic gradient descent for logistic regression on MNIST dataset, empirical value iteration for discounted and average cost MDPs, and empirical Q-value iteration. 

\subsection{Stochastic Gradient Descent}
Consider the task of classifying a subset of MNIST handwritten digits, where we consider only the images corresponding to the numbers 0 and 1. Each data point is a $28\times 28$ pixel image with the corresponding label (either $0$ or $1$). We use logistic regression and Poisson regression with an $\ell_2$ regularizer for the classification task. We refer the reader to Subsection \ref{sub:logistic} for details of this problem for the logistic regression. The loss function for the logistic regression with the regularizer is
\beqq{L_i(x):= -l^i\log f(u^i;x) - (1-l^i)\log(1-f(u^i;x)) +\frac{\lambda}{2}\|x\|^2_2, \text{ where } \lambda = 5.}
The loss function for Poisson regression with regularizer is
\begin{align*}
L(x) = {1\over N}\sum_{i=1}^NL_i(x), \text{ where } L_i(x) = \exp(u^{i\transpose}x) - l^i(u^{i\transpose}x) + \frac{\lambda}{2}\|x\|_2^2, \text{ where } \lambda = 1.
\end{align*}
The first and the second derivatives of this loss function is 
\begin{align*}
\nabla_xL_i(x) &= \exp(u^{i\transpose}x)u^i - l^iu^i + \lambda x,\qquad    \nabla_{xx}^2L_i(x) = \exp(u^{i\transpose}x)u^iu^{i\transpose} +\lambda I
\end{align*}

\begin{figure}[bth]
\includegraphics[width=\textwidth,trim={0 0.5cm 0 0.5cm},clip]{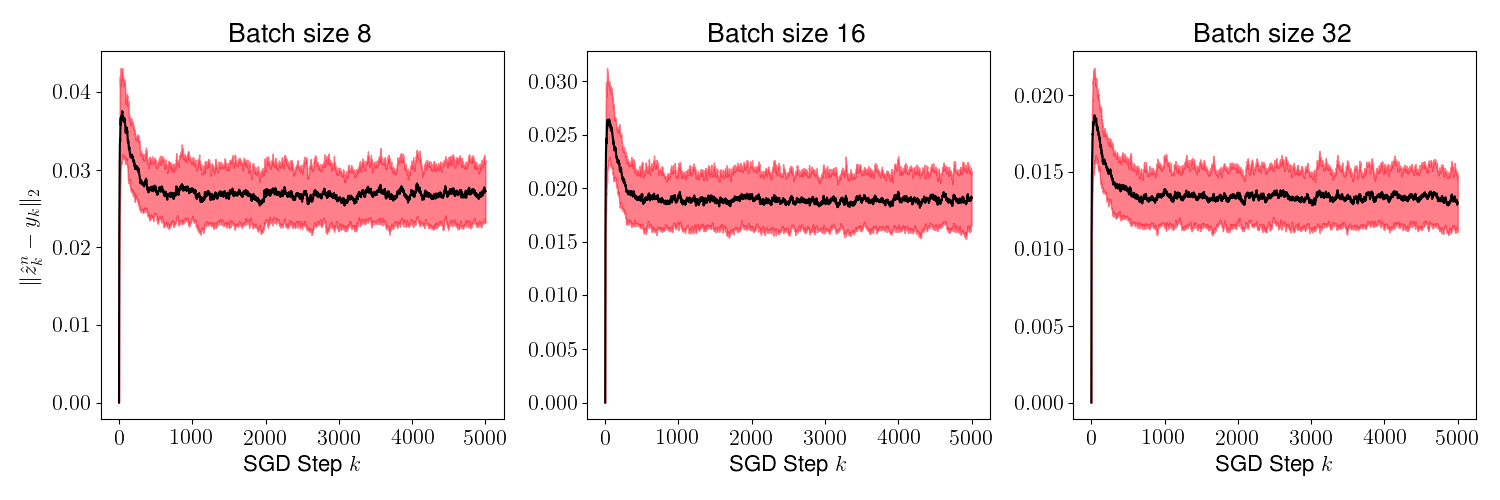}
\vspace*{-2.2em}
\caption{The Euclidean norm between the outputs of the exact gradient descent and the stochastic gradient descent with $n=8,16,32$ for the logistic regression with a regularizer. We observe that as $n$ grows, the probability that $\|\hat z^n_k-y_k\|$ is large becomes smaller. The mean $\ex{\|\hat z^n_k-y_k\|}$ (showed using black line) and the variance of $\|\hat z^n_k-y_k\|$ are computed using 1000 independent runs of the iterations.}
\label{fig:log}
\end{figure}

\begin{figure}[bth]
\includegraphics[width=\textwidth]{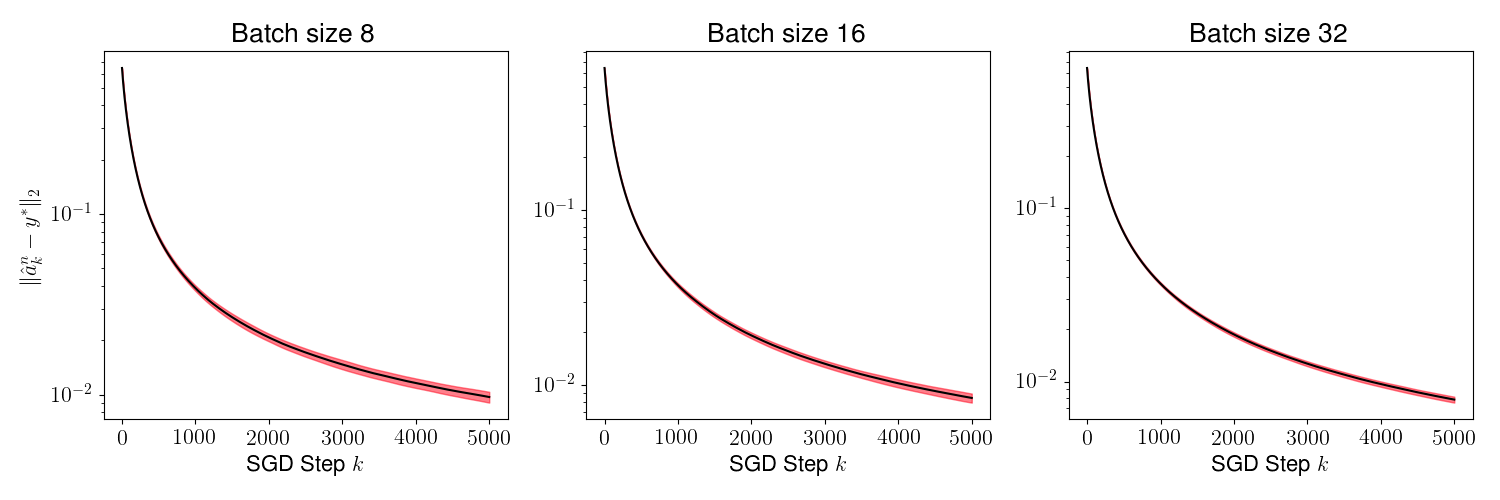}
\vspace*{-2.2em}
\caption{Plot of $\|\hat a^n_k-x^*\|$ for $n=8,16,32$ for the logistic regression with a regularizer. The variance of $\|\hat a^n_k-x^*\|$ is decreasing as $k$ grows, which indicates that the variance reduction property of the time averages. The mean and the variance of $\|\hat a^n_k-x^*\|$ are computed using 1000 independent runs of the stochastic gradient descent iterations.}
\label{fig:ta_log}
\end{figure}

\begin{figure}[bth]
\includegraphics[width=\textwidth]{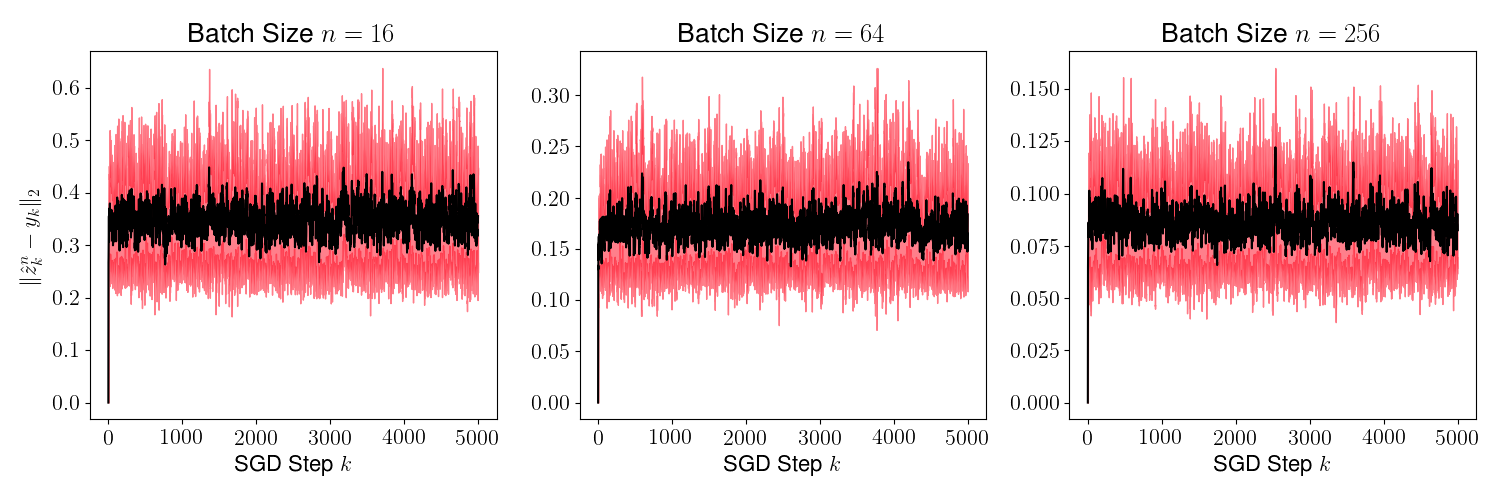}
\vspace*{-2.2em}
\caption{The Euclidean norm between the outputs of the exact gradient descent and the stochastic gradient descent with $n=64,256,1024$ for the Poisson regression with a regularizer. We observe that as $n$ grows, the probability that $\|\hat z^n_k-y_k\|$ is large becomes smaller. The mean $\ex{\|\hat z^n_k-y_k\|}$ (showed using black line) and the variance of $\|\hat z^n_k-y_k\|$ are computed using 1000 independent runs of the iterations.}
\label{fig:pi}
\end{figure}

\begin{figure}[bth]
\includegraphics[width=\textwidth]{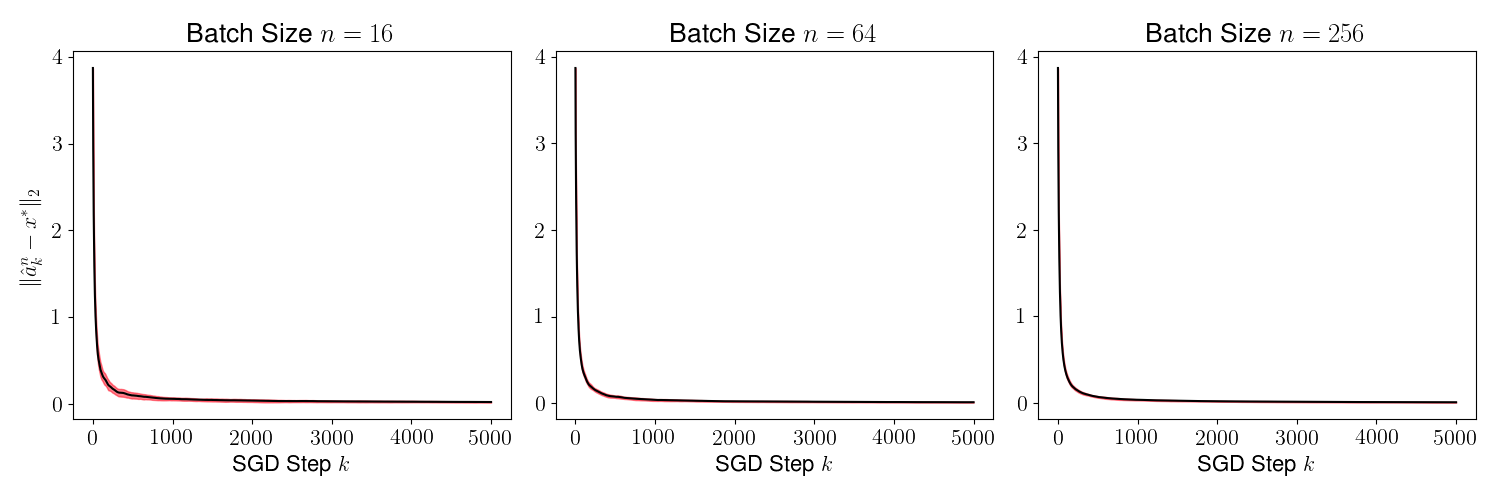}
\vspace*{-2.2em}
\caption{Plot of $\|\hat a^n_k-x^*\|$ for $n=16,64,256$ for the Poisson regression with a regularizer. The variance of $\|\hat a^n_k-x^*\|$ is decreasing as $k$ grows, which indicates that the variance reduction property of the time averages. The mean and the variance of $\|\hat a^n_k-x^*\|$ are computed using 1000 independent runs of the stochastic gradient descent iterations.}
\label{fig:ta_pi}
\end{figure}

We use here the notations introduced in Subsection \ref{sub:logistic}. We transform each image into a vector and append $1$ at the beginning of the vector. Thus, the space $\ALP U = \{1\}\times[0,1]^{784}$. Thus, the space $\ALP X = \Re^{785}$. As mentioned previously, the variable $n$ represents the batch size picked at every SGD iteration step. We pick $y_0$ arbitrarily in $\ALP X$ and set $\hat z^n_0 = y_0$. Then, we run the exact gradient descent and the minibatch SGD as follows:
\beqq{y_{k+1} = T(y_k),\quad \hat z^n_{k+1} = \hat T^n_k(\hat z^n_k), \quad \hat a^n_k = \frac{1}{k+1}\sum_{i=0}^k \hat z^n_i.}
As discussed in Subsection \ref{sub:logistic}, it is clear that the exact gradient descent is a contraction map for stepsize $\beta$ small enough. Therefore, $y_k$ converges to the optimal solution $x^*$. We can make the following claim about the operator $\hat T^n_k$:
\begin{theorem}
The random operator $\hat T^n_k$ satisfies Assumptions \ref{assm:hatt}, \ref{assm:suppr}. If $\beta$ is sufficiently small, then $\hat T^n_k$ also satisfies Assumption \ref{assm:contraction}. Let $\mu_n$ denote the distribution of $(\hat z^n_0,\hat z^n_1,\ldots)$ and $\psi$ be the Dirac mass on $(y_0,y_1,\ldots)$. We have $\{\mu_n\}_{n\in\Na}$ converges to $\psi$ in weak* topology as $n\rightarrow\infty$ and there exists a unique invariant distribution of the Markov chain $(\hat z^n_k)_{k\in\Na}$ for any $n\in\Na$. 
\end{theorem}
\begin{proof}
The first conclusion follows from the discussion leading to \eqref{eqn:logeps} in Subsection \ref{sub:logistic}. For $\beta$ sufficiently small, every realization of $\hat T^n_k$ is a contraction as discussed in Remark \ref{rem:hatt}. We only need to show that Part 2 of Assumption \ref{assm:contraction} is satisfied. Since $n$ is finite, $\hat T^n_k(x^*)$ can take only finitely many values, and therefore Part 2 of Assumption \ref{assm:contraction} is trivially satisfied. The existence of invariant distribution then follows from Theorem \ref{thm:invariant}.
\end{proof}

As a result of the theorem above, we conclude that $\|y_k-\hat z^n_k\|_2$ is generally small for most $k$, and its variance converges to 0 as $n$ increases. This is evident in Figures \ref{fig:log} and \ref{fig:pi}, where we plot the distance $\|y_k-\hat z^n_k\|_2$ for various values of $n$ (batch sizes) and $k$ ranging from 0 to 5000. The black curve is the mean $\ex{\|y_k-\hat z^n_k\|_2}$ and the red region shows one standard deviation of the distance $\|y_k-\hat z^n_k\|_2$ over 1000 sample paths of the minibatch SGD iterations. We picked the same initial condition $\hat z^n_0 = y_0 = 0$ for all the sample paths to generate the figure.

Figures \ref{fig:ta_log} and \ref{fig:ta_pi} plots $\|\hat a^n_k- x^*\|_2$, and we observe that the variance of time averages reduces as $k$ grows large for any batch size $n$. This variance reduction of time averages can be attributed to the contraction property of the random maps $\hat T^n_k$, which in turn is due to the addition of a strongly convex regularizer to the empirical loss function.


\subsection{Empirical Value Iteration for Discounted Cost MDP}

We consider here the empirical value iteration for discounted Markov decision processes (MDP) as described in Subsection \ref{sub:edpdiscount}. Consider the value iteration algorithm applied to an MDP in which there are 20 states and 5 actions. We generate the state transition probability matrix for this MDP randomly at the beginning of the simulation.

We use here the notation introduced in Subsection \ref{sub:edpdiscount}. We initialize $v_0$ arbitrarily and set $\hat z^n_0 = v_0$, and define the iterates of exact value iteration and empirical value iteration as
\beqq{v_{k+1} = T(v_k),\qquad \hat z^n_{k+1} = \hat T^n_k(\hat z^n_k), \quad \hat a^n_k = \frac{1}{k+1}\sum_{i=0}^k \hat z^n_i.}

\begin{figure}[bth]
\includegraphics[width=\textwidth]{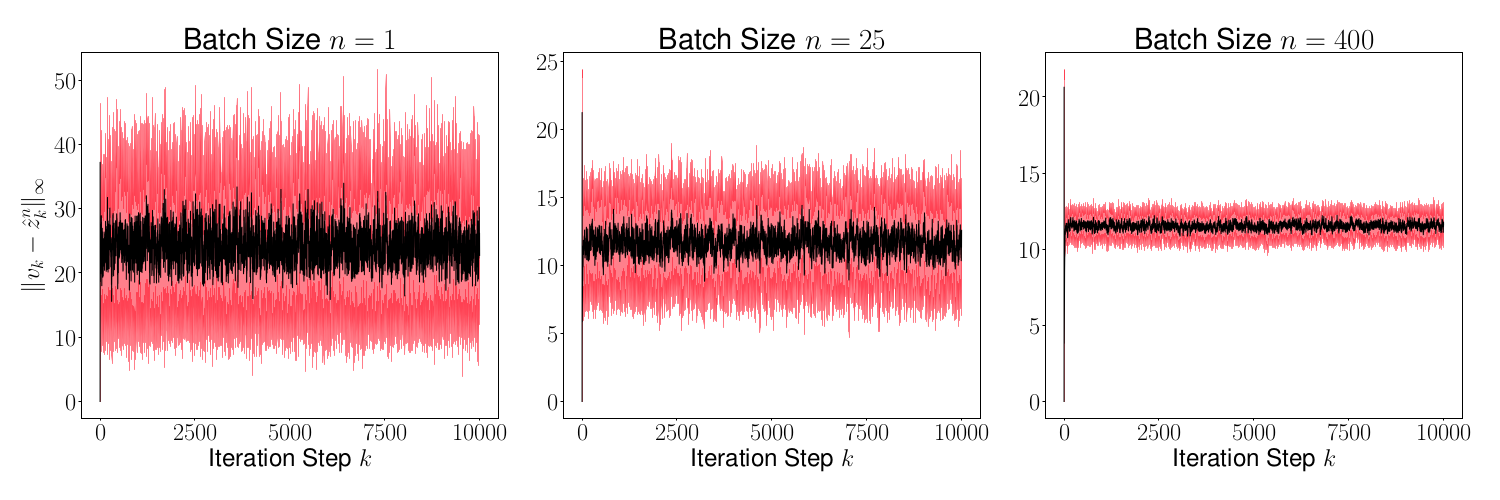}
\caption{\label{fig:valiter}Plot of $\|v_k - \hat z^n_k\|$ for $n = 1,25,400$ for $k=1,\ldots,1000$. It is clear from the plots that as $n$ grows, the average and variance of $\|v_k - \hat z^n_k\|$ reduces. The mean and the variance of $\|v_k - \hat z^n_k\|$ are computed using 1000 independent runs of the iterations. There are 20 states and 5 actions in this MDP.}
\end{figure}

\begin{figure}[bth]
\includegraphics[width=\textwidth]{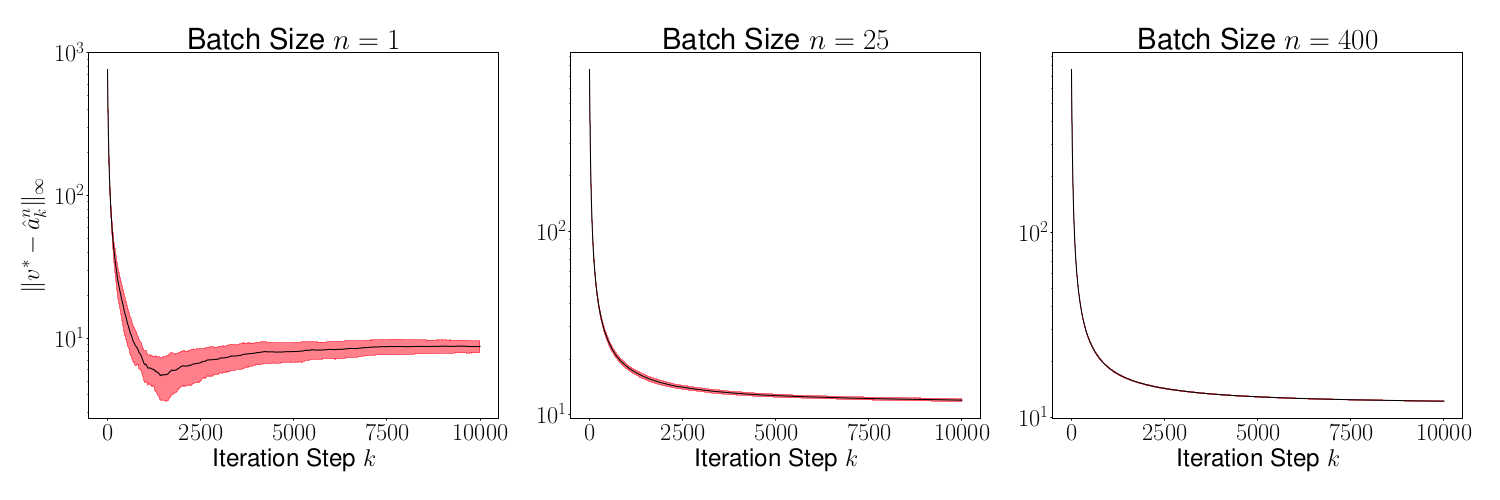}
\caption{\label{fig:ta_valiter} Plot of $\|v^* - \hat a^n_k\|$ for $n = 1,25,400$ for $k=1,\ldots,1000$. Notice that for every $n$, the variance in $\|v^* - \hat z^n_k\|$ reduces as $k$ increases. The plots are constructed using 1000 sample paths. The mean and the variance of $\|v^* - \hat a^n_k\|$ are computed using 1000 independent runs of the iterations.}
\end{figure}
We can prove the following result.
\begin{theorem}\label{thm:edpdiscount}
The random operator $\hat T^n_k$ satisfies Assumptions \ref{assm:hatt}, \ref{assm:suppr}, and \ref{assm:contraction}. As a result, we conclude:
\begin{enumerate}[leftmargin=1em]
 \item Let $\mu_n$ denote the distribution of $(\hat z^n_0,\hat z^n_1,\ldots)$ and $\psi$ be the Dirac mass on $(y_0,y_1,\ldots)$. We have $\{\mu_n\}_{n\in\Na}$ converges to $\psi$ in weak* topology as $n\rightarrow\infty$.
 \item There exists a unique invariant distribution $\pi^n$ of the Markov chain $(\hat z^n_0,\hat z^n_1,\ldots)$. Further, the sequence of distribution of $\hat z^n_k$ converges to this invariant distribution $\pi^n$.
 \item The time average $\hat a^n_k$ converges in probability to the mean of $\pi^n$ as $k\rightarrow\infty$.
\end{enumerate}
\end{theorem}
\begin{proof}
The proof follows from the discussions in Subsection \ref{sub:edpdiscount}, where we show that $\hat T^n_k$ satisfies Assumptions \ref{assm:suppr} and \ref{assm:contraction}.

Moreover, since every realization of $\hat T^n_k$ is a contraction operator with coefficient $\alpha$, if $\|v\|_\infty\leq \frac{\|c\|_\infty }{1-\alpha}$, then $\|\hat T^n_k(v)\|_\infty \leq  \frac{\|c\|_\infty }{1-\alpha}$. Thus, there is no loss of generality in restricting $\ALP X$ to be in the compact set $\{v\in\Re^{|\ALP S|}: \|v\|_\infty\leq \frac{\|c\|_\infty }{1-\alpha}\}$. The fact that the time averages converge in probability to the mean of the invariant distribution is a direct application of Theorem \ref{thm:acompact}.
\end{proof}

Figure \ref{fig:valiter} shows the mean (black line) and the variance (red region) of $\|v_k - \hat z^n_k\|_\infty$ for every time step $k = 1,\ldots,10000$ for $n = 1, 25, 400$. To compute the mean and the variance, 1000 sample paths were taken. As proved in Theorem \ref{thm:edpdiscount} above, the mean of $\|v_k - \hat z^n_k\|_\infty$ becomes smaller as $n$ grows. Figure \ref{fig:ta_valiter} shows the mean and the variance of $\|v^* - \hat a^n_k\|_\infty$ for $n = 1,25,400$ for $k=1,\ldots,10000$. For every $n$, since the sequence of distribution of $\hat z^n_k$ converges to a unique invariant distribution, $\hat a^n_k$ converges in probability to the mean of the invariant distribution. We observe in the figure that the variance of $\|v^* - \hat a^n_k\|_\infty$ reduces as $k$ grows implying that the average $\hat a^n_k$ is close but not equal to $v^*$. We further observe that the variance of $\|v^* - \hat a^n_k\|_\infty$ is vanishingly small for large values of $n$ -- a result that we have not formally proved here, and can be addressed in a future work.

\subsection{Synchronous Batch Q-Value Iteration for Discounted Cost MDP}
Q-value iteration is another algorithm that, like value iteration, computes the optimal value function in MDPs. Let $\ALP Q$ denote the set of all Q-value functions $Q: \mathcal{S}\times\mathcal{A}\rightarrow \Re$. Similar to Bellman operator of value iteration, we define an operator $T:\ALP Q\rightarrow\ALP Q$ as
\begin{equation*}
    T(q)(s,a) = c(s,a) + \alpha \sum_{s'\in\ALP S}p(s'|s,a)\min_{a'\in\ALP A} q(s',a').
\end{equation*}
Similar to Bellman operator, $T$ is a contraction on $(\mathcal{Q}, \|\cdot\|_\infty)$. Further, it can be shown that the fixed point of $T$ is $q^*$, which is defined as $q^*(s,a) = c(s,a)+\alpha \ex{v^*(s')|s,a}$, where $v^*(\cdot) = \min_{a\in\ALP A} q(\cdot,a)$. The Q-value iteration starts with an arbitrary $q_0\in\ALP Q$ and generates the sequence according to $q_{k+1} = T(q_k)$, which converges to $q^*$ as $k\rightarrow\infty$. 

The exact operator, as in other cases considered in the paper, can be approximated by the empirical operator $\hat T^n_k$:
\begin{equation*}
    \hat T^n_k(q)(s, a) = c(s, a) + \alpha \frac{1}{n}\sum_{i=1}^n \min_{a'\in\ALP A} q(s'_{k,i}(s,a),a'),
\end{equation*}
where $\{s'_{k,i}(s,a)\}_{i=1}^n$ are $n$ i.i.d. samples of the next state given the current state-action pair $(s,a)$. 

Let us define $\hat z^n_{k+1} = \hat T^n_k(\hat z^n_k)$, where $\hat z^n_0 = q_0$. Let $\hat a^n_k$ be the time averaged version of $\hat z^n_k$. The properties of the random operator $\hat T^n_k$ for empirical Q-value iteration has the same properties as listed in Theorem \ref{thm:edpdiscount} for the case of empirical value iteration. Thus, we omit repetition of essentially the same result here. The simulation results are plotted in Figure \ref{fig:qvaliter} and \ref{fig:ta_qvaliter}.

\begin{figure}[bth]
\includegraphics[width=\textwidth]{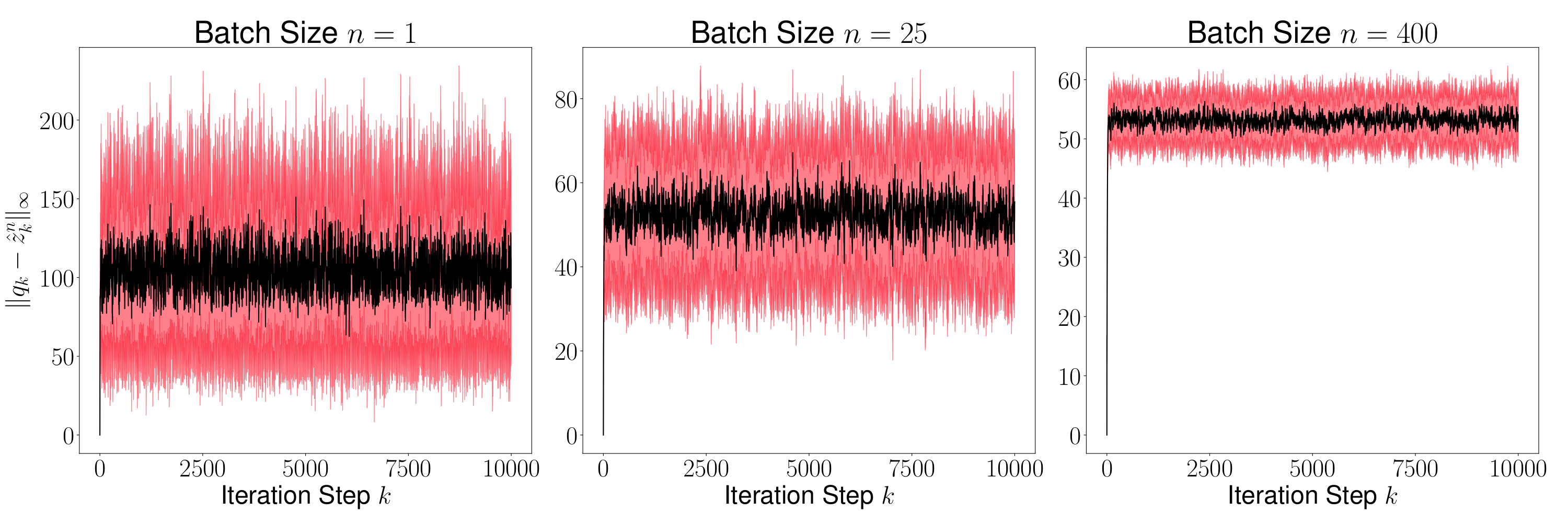}\vspace*{-2.5em}
\caption{Plot of the mean and the variance of $\|q_k - \hat z^n_k\|$ for $n = 1,25,400$ for $k=1,\ldots,1000$. It is clear from the plots that as $n$ grows, the mean of $\|q_k - \hat z^n_k\|$ reduces. The mean and the variance of $\|q_k - \hat z^n_k\|$ are computed using 1000 independent runs of the iterations. There are 20 states and 5 actions in this MDP.}
\label{fig:qvaliter}
\end{figure}
\begin{figure}[bth]
\includegraphics[width=\textwidth]{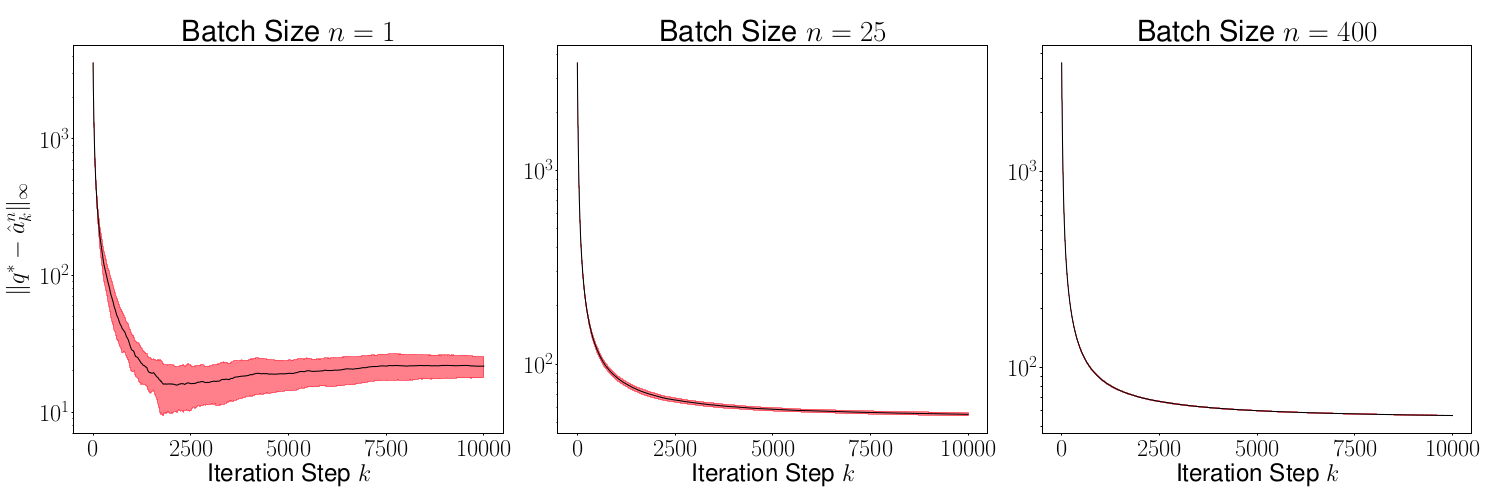}\vspace*{-2.5em}
\caption{Plot of the mean and the variance of $\|q^* - \hat a^n_k\|$ for $n = 1,25,400$ for $k=1,\ldots,1000$. Notice that for every $n$, the variance in $\|q^* - \hat z^n_k\|$ reduces as $k$ increases. The mean and the variance of $\|q^* - \hat a^n_k\|$ are computed using 1000 independent runs of the iterations starting from the same initial condition $\hat z^n_0 = q_0$.}
\label{fig:ta_qvaliter}
\end{figure}

Figure \ref{fig:qvaliter} shows the mean (black line) and the variance (red region around the mean) of  $\|q_k - \hat z^n_k\|_\infty$ for various values of $n$ for $k$ ranging from $1$ to $10000$. This mean and variance is computed using 1000 sample paths of the algorithm starting from the same initial $\hat z^n_0 = q_0$ for all sample paths. As is evident, the mean reduces as we increase the sample size $n$.

Figure \ref{fig:ta_qvaliter} plots the mean and the variance of $\|q^* - \hat a^n_k\|$. Once again, we observe that the variance reduces as $k$ increases since the Markov chain admits a unique invariant measure.

\section{Proof of Theorem \ref{thm:main}}\label{sec:mainproof}
To prove Theorem \ref{thm:main}, we need to introduce some further notation. Define $\Z = \{0,1,2,\ldots\}$ and we use $\mu_n\Rightarrow\psi$ to denote the convergence in weak topology. Let $\Pi_k:\ALP X^\Z\rightarrow\ALP X^{k+1}$ denote the projection operator that projects a sequence to its first $(k+1)$ components, that is,
\beqq{\Pi_k(x_0,x_1,x_2,\ldots) = (x_0,\ldots,x_k), \quad k\in \Z.}
For a measure $\mu\in\wp(\ALP X^\Z)$, let $\mu\circ \Pi_k^{-1}\in\wp(\ALP X^{k+1})$ denote the pullback of the measure to the first $k+1$ components. We note the following fact from probability theory.
\begin{theorem}[\cite{billingsley2013convergence}, Theorem 2.8]
Let $(\nu_n)_{n\in\Na}\subset\wp(\ALP X^\Z)$ be a sequence of measures. Then, $(\nu_n)_{n\in\Na}\Rightarrow\nu_\infty$ if and only if $\nu_n\circ \Pi_k^{-1}\Rightarrow\nu_\infty \circ \Pi_k^{-1}$ for every $k\in\Z$.  
\end{theorem}

Recall that $\mu_n$ is a measure defined on a stochastic sequence and we would like to show that $\mu_n\Rightarrow\psi$. Due to the theorem above, all we need to establish is that $\mu_n\circ \Pi_k^{-1}\Rightarrow\psi\circ \Pi_k^{-1}$ for every $k$. We will establish this result via an induction argument following \cite{karr1975}. 

Fix $n\in\Na$. For $k = 0$, let $\mu_n\circ \Pi_0^{-1}(F) = \ind{y_0}(F) = \psi\circ\Pi_0^{-1}(F)$ for all measurable $F\subset\ALP X$. Thus, the statement is true for $k = 0$. In the induction step, we prove that if $\mu_n\circ\Pi_{k-1}^{-1}\Rightarrow \psi\circ\Pi_{k-1}^{-1}$ and for every closed set $F_1\subset\ALP X^k$ and $F_2\subset\ALP X$, we have 
\beqq{\limsup_{n\rightarrow\infty}\mu_n\circ\Pi_k^{-1}(F_1\times F_2)\leq\psi\circ\Pi_k^{-1}(F_1\times F_2).}
By the Portmanteau theorem \cite[Theorem 2.1]{billingsley2013convergence}, we then conclude that $\mu_n\circ\Pi_k^{-1}\Rightarrow \psi\circ\Pi_k^{-1}$ for every $k\in\Z$. We divide the proof into three steps.

{\it Step 1:} We show that for any uniformly continuous function $g$, the map $g\circ T$ is also uniformly continuous since $T$ is Lipschitz operator.
\begin{lemma}\label{lem:gunif}
If $g\in U_b(\ALP X)$, then $g\circ T \in U_b(\ALP X)$.
\end{lemma}
\begin{proof}
Fix $\epsilon>0$. Since $g\in U_b(\ALP X)$, we can pick a $\delta_\epsilon>0$ such that for any $x,x'\in\ALP X$ with $\rho(x,x')<\delta_\epsilon$, $|g(x)-g(x')|<\epsilon$. Now, for any $y,y'\in\ALP X$ with $\rho(y,y')<\delta_\epsilon$, we know that
\beqq{\rho(T(y),T(y')) <\alpha\delta_\epsilon<\delta_\epsilon.}
Taking $x = T(y)$ and $x' = T(y')$, we conclude that
\beqq{|g\circ T(y)- g\circ T(y') |<\epsilon.}
This implies that $g\circ T\in U_b(\ALP X)$.
\end{proof}

{\it Step 2:} Next, we establish a consequence of Assumption \ref{assm:suppr} and Lemma \ref{lem:gunif}. Let $x_{0:l} = (x_0,\ldots,x_l)$. We assume that $\mu_n\circ \Pi_k^{-1}\Rightarrow\psi\circ \Pi_k^{-1}$ for all $k=0,\ldots,l$, which serves as the induction hypothesis.
\begin{lemma}\label{lem:weak}
Suppose that Assumption \ref{assm:suppr} holds. Then, for every $\epsilon>0$, there exists an $N_\epsilon$ such that
\beqq{\Bigg|\int_{F_1} g(T(x_l)) \mu_n\circ \Pi_l^{-1}(dx_{0:l}) - \int_{F_1} g(T(x_l))\psi\circ \Pi_l^{-1}(dx_{0:l})\Bigg| <\epsilon \text{ for all } n\geq N_\epsilon,}
and 
\beqq{ \int_{\ALP X}  \Big|\ex{g(\hat T^n_l(x_l)} - g(T(x_l))\Big| \mu_n\circ \Pi_l^{-1}(dx_{0:l}) \leq \epsilon \qquad \text{ for all } n\geq N_\epsilon.}
\end{lemma}
\begin{proof}
Lemma \ref{lem:gunif} implies that $g\circ T$ is uniformly continuous. The first result is a direct consequence of the induction hypothesis. The proof of the second result is presented in Appendix \ref{app:weak}.
\end{proof}

{\it Step 3:} We now complete the proof by establishing the induction step using the result from Lemma \ref{lem:weak}. We again follow \cite{karr1975}. Let $F_1\subset\ALP X^{l+1}$ be any closed set. Then, using Lemma \ref{lem:weak}, for any $g\in U_b(\ALP X)$, there exists $N_\epsilon$ such that
\beq{& \int_{F_1} \int_{\ALP X} g(x_{l+1})\mu_n\circ \Pi_{l+1}^{-1}(x_{0:l+1}) - \int_{F_1} \int_{\ALP X} g(x_{l+1})\psi\circ \Pi_{l+1}^{-1}(dx_{0:l+1}) \nonumber\\
& = \int_{F_1}  \ex{g(\hat T^n_l(x_l))}\mu_n\circ \Pi_l^{-1}(dx_{0:l}) - \int_{F_1} g(T(x_l))\psi\circ \Pi_l^{-1}(dx_{0:l}) \nonumber \\
& \leq  \int_{\ALP X^{l+1}}  \Big|\ex{g(\hat T^n_l(x_l)} - g(T(x_l))\Big| \mu_n\circ \Pi_l^{-1}(dx_{0:l})\nonumber\\
& + \Bigg|\int_{F_1} g(T(x_l)) \mu_n\circ \Pi_l^{-1}(dx_{0:l}) - \int_{F_1} g(T(x_l))\psi\circ \Pi_l^{-1}(dx_{0:l})\Bigg| \nonumber\\
& <2\epsilon \text{ for all } n\geq N_\epsilon\label{eqn:munpsi}}
Now, for any $F_2\subset\ALP X$ closed, we can construct a sequence of $(g_m)_{m\in\Na}\subset U_b(\ALP X)$ such that $g_{m+1}\leq g_m$ for all $m\in\Na$ and $g_m\rightarrow 1_{F_2}$. This leads us to the following inequality for every $m\in\Na$:
\beqq{\int_{F_1}\int_{\ALP X} 1_{F_2}(x_l)\mu_n\circ \Pi_{l+1}^{-1}(x_{0:l+1}) \leq \int_{F_1}\int_{\ALP X} g_m(x_l)\mu_n\circ \Pi_{l+1}^{-1}(x_{0:l+1}).}
Taking the limsup on both sides and using \eqref{eqn:munpsi}, we arrive at the following inequality
\beq{\label{eqn:limsupmun}\limsup_{n\rightarrow\infty}\mu_n\circ \Pi_{l+1}^{-1}(F_1\times F_2) \leq \int_{F_1} \int_{\ALP X} g_m(x_{l+1})\psi\circ \Pi_{l+1}^{-1}(dx_{0:l+1}).}
Now, since the right side holds for every $m\in\Na$, we take the limit $m\rightarrow\infty$ and use the bounded convergence theorem to conclude
\beq{\label{eqn:limmpsi}\lim_{m\rightarrow\infty}\int_{F_1} \int_{\ALP X} g_m(x_{l+1})\psi\circ \Pi_{l+1}^{-1}(dx_{0:l+1}) = \psi\circ \Pi_{l+1}^{-1}(F_1\times F_2).}
Collecting the two inequalities in \eqref{eqn:limsupmun} and \eqref{eqn:limmpsi}, we conclude that
\beqq{ \limsup_{n\rightarrow\infty}\mu_n\circ \Pi_{l+1}^{-1}(F_1\times F_2) \leq \psi\circ \Pi_{l+1}^{-1}(F_1\times F_2).}
The proof of the lemma is complete.

\section{Proof of Theorem \ref{thm:lln}}\label{sec:llnproof}
We first introduce some notations. For $f\in C_b(\ALP X)$, we define $g_k\in C_b(\ALP X)$, $k=0,1,\ldots$ as
\beqq{g_0(x) = f(x), \quad g_k(x) = F^k(f)(x) = \ex{f\left( \hat T^n_k\circ\ldots\circ \hat T^n_1(x)\right)}.}
The following equation follows immediately from the above definitions:
\beq{\label{eqn:gkm}\ex{g_k(\hat z^n_m)|\hat z^n_{m-1}} = g_{k+1}(\hat z^n_{m-1}).}
Define $\bar F^*_k = \frac{1}{k}\sum_{i=1}^{k} (F^*)^i$. Define the constant function $c_f\in C_b(\ALP X)$ as 
\beqq{c_f(x) =  \int f d\pi^n.}
The average of the functions $g_k$, denoted by $\bar g_k$, is
\beqq{\bar g_k(x) = \frac{1}{k}\sum_{i=1}^k g_{i-1}(x).}
For a function $f\in C_b(\ALP X)$ and a set $C\subset\ALP X$, we use $f|_C$ to denote the restriction of the function on the set $C$. We now prove three lemmas that lead to the result. For the next result, let us define the occupation measure $\eta_k$ over the set $C\subset\ALP X$ as 
\beq{\label{eqn:etak}\eta_k(C) = \frac{1}{k}\sum_{i=0}^{k-1} 1_{\{\hat z^n_k\in C\}}.}
We claim the following.
\begin{lemma}\label{lem:occ}
If Assumptions \ref{assm:hatt} and \ref{assm:unique} holds, then for every $\epsilon>0$, there exists a compact set $C_\epsilon\subset\ALP X$ such that 
\beqq{\limsup_{k\rightarrow\infty}\pr{\eta_k(C_\epsilon^\complement)\geq \epsilon} <\epsilon.}
\end{lemma}
\begin{proof}
Note that since $\bar F^*_k(\mu)\rightarrow\pi^n$ for any $\mu\in\ALP M$ by Assumption \ref{assm:unique}, we conclude from Prohorov's theorem that the sequence $\{\bar F^*_k(\mu)\}_{k\in\Na}$ is tight. Thus, for $\epsilon>0$, let $C_\epsilon$ be the compact set such that
\beqq{\bar F^*_k(\mu)(C_\epsilon^\complement) <\epsilon^2 \text{ for all } k\in\Na.}
Further, we note that for any set $C\subset\ALP X$, we have
\beqq{\ex{\eta_k(C)} = \ex{\frac{1}{k}\sum_{i=0}^{k-1} 1_{\{\hat z^n_k\in C\}}} = \bar F^*_k(\mu)(C).}
Using the above identity and using Markov's inequality, we conclude that
\beqq{\pr{\eta_k(C_\epsilon^\complement)\geq \epsilon} \leq \frac{\ex{\eta_k(C_\epsilon^\complement)}}{\epsilon} <\epsilon. }
The proof of the theorem is complete.
\end{proof}

\begin{lemma}\label{lem:equi}
Let $C\subset\ALP X$ be a compact set. Then, the sequence of functions $(\bar g_k|_C)_{k\in\Na}$ is uniformly bounded and equicontinuous and converges uniformly to $c_f|_C$. 
\end{lemma}
\begin{proof}
First, we note that $\|g_i\|_\infty\leq \|f\|_\infty$ for all $i\in\Na$, which implies that $\bar g_k|_C$ is uniformly bounded.

The proof of equicontinuity follows directly from Assumption \ref{assm:unique}(2) and Ascoli theorem. Note that as $k\rightarrow\infty$, we get
\beqq{\bar g_k(x) = \langle f,\bar F^*_k(\delta_x)\rangle \rightarrow c_f(x).}
Thus, by Ascoli's theorem, $\{\bar g_k|_C\}_{k\in\Na}$ is an equicontinuous sequence of functions. The result then follows using Assumption \ref{assm:unique}(2).
\end{proof}

Next, we establish the following result.
\begin{lemma}\label{lem:gkN}
For every $M\in\Na$, we have
\beqq{\lim_{N\rightarrow\infty}\left|\frac{1}{N} \sum_{l=0}^{N-1} \Big(g_0(\hat z^n_l) - \bar g_M(\hat z^n_l)\Big)\right| = 0 \quad \mathbb{P}\text{-almost surely}.}
\end{lemma}
\begin{proof}
See Appendix \ref{app:gkN}.
\end{proof}
The proof can now be completed easily. Fix $\epsilon>0$ and recall the definition of the set $C_\epsilon$ from Lemma \ref{lem:occ}. We now note that for every $K\in\Na$ and $M\in\Na$, we have
\beqq{& \frac{1}{K}\sum_{k=1}^K f(\hat z^n_k) - \int fd\pi^n = \frac{1}{K}\sum_{k=1}^K g_0(\hat z^n_k) - \int fd\pi^n\\
& \leq   1{\{  \hat z^n_1 \in C_\epsilon,\ldots,\hat z^n_K \in C_\epsilon \}}\left( \frac{1}{K}\sum_{k=1}^K g_0(\hat z^n_k) - \int fd\pi^n \right)  + 2\|f\|_\infty\eta_K(C_\epsilon^\complement),\\
& < 1{\{\hat z^n_1 \in C_\epsilon,\ldots,\hat z^n_K \in C_\epsilon \}}\left\{ \frac{1}{K}\sum_{k=1}^K \Big(g_0(\hat z^n_k) - \bar g_M(\hat z^n_k)\Big) + \left( \frac{1}{K}\sum_{k=1}^K \bar g_M(\hat z^n_k) - \int fd\pi^n \right)\right\} \\
& \qquad + 2\|f\|_\infty\eta_K(C_\epsilon^\complement).}
Since $\bar g_M(x) \rightarrow\int fd\pi^n $ uniformly on the compact set $C_\epsilon$ due to Lemma \ref{lem:equi}, we can pick $M_\epsilon$ sufficiently large such that for all $K\in\Na$ and $M\geq M_\epsilon$, we get
\beqq{1{\{\hat z^n_0 \in C_\epsilon,\ldots,\hat z^n_K \in C_\epsilon \}}\left( \frac{1}{K}\sum_{k=1}^K \bar g_M(\hat z^n_k) - \int fd\pi^n \right) <\epsilon.}
For such $M_\epsilon$, as $K\rightarrow\infty$, the first summand goes to 0 by Lemma \ref{lem:gkN}. In the third summand, we know that $\eta_K(C_\epsilon^\complement)$ is less than $\epsilon$ with probability at least $1-\epsilon$ due to Lemma \ref{lem:occ} for sufficiently large $K$. Collecting all these results, we conclude that
\beqq{\limsup_{K\rightarrow\infty} \pr{\Bigg|\frac{1}{K}\sum_{k=1}^K f(\hat z^n_k) - \int fd\pi^n\Bigg|<(2\|f\|_\infty+2)\epsilon}\geq 1-\epsilon.}
This completes the proof of the theorem.

\subsection{Proof of Theorem \ref{thm:bounded}}\label{sub:bounded}
In order to prove the result, let us first consider the function $f:\ALP X\rightarrow[0,\infty)$, which always takes non-negative values. Let $f_m(x):=\min\{f(x),m\}$ be the clipped function, in which case, $f_m\in C_b(\ALP X)$. Define $\Omega_0 = \{\omega:C(\omega)<\infty\}$, and note that by assumption, $\pr{\Omega_0} = 1$. It is obvious that for any $x\in\ALP X$, if we pick $m\geq f(x)$, then $f_m(x) = f(x)$. 

Note that for this case, for any $m\in\Na$, the following inequality holds:
\beq{\Bigg|\frac{1}{K}\sum_{k=0}^{K-1}f(\hat z^n_k) - \int fd\pi^n\Bigg|\leq & \Bigg|\frac{1}{K}\sum_{k=0}^{K-1}\Big( f(\hat z^n_k) - f_m(\hat z^n_k)\Big)\Bigg| + \Bigg|\frac{1}{K}\sum_{k=0}^{K-1}f_m(\hat z^n_k) - \int f_md\pi^n\Bigg| \nonumber\\
& + \Bigg| \int f_md\pi^n-\int fd\pi^n \Bigg|.\label{eqn:kmdiff}}
In the next lemma, we show that each of the summand on the right is small with high probability. 
\begin{lemma}\label{lem:positive}
Let $f:\ALP X\rightarrow[0,\infty)$. If Assumptions \ref{assm:hatt}, \ref{assm:unique}, and \ref{assm:bounded} hold, then \eqref{eqn:asconv} holds.
\end{lemma}
\begin{proof}
For any $\epsilon>0$, there exists (a random natural number) $M_\epsilon(\omega)$ such that 
\beqq{\Bigg|\frac{1}{K}\sum_{k=0}^{K-1}\Big( f(\hat z^n_k) - f_m(\hat z^n_k)\Big)\Bigg|<\epsilon \text{ for all } m\geq M_\epsilon(\omega) \text{ and } K\in\Na.}
Indeed, one can take $M_\epsilon(\omega) = \lceil C(\omega)\rceil$, in which case, 
\beq{\label{eqn:mooreosgood}\Bigg|\frac{1}{K}\sum_{k=0}^{K-1}\Big( f(\hat z^n_k) - f_m(\hat z^n_k)\Big)\Bigg|=0 \text{ for all } m\geq M_\epsilon(\omega) \text{ and } K\in\Na.}
Further, due to the monotone convergence theorem, there exists $\bar M$ such that
\beqq{\left|\int f_m d\pi^n - \int fd\pi^n\right|<\epsilon \text{ for all } m\geq \bar M.}
Now, pick $\epsilon,\delta>0$. Let us define $\Omega_{m,K}\subset\Omega$ as the set such that
\beqq{\Omega_{m,K} = \left\{\omega: \left|\frac{1}{K}\sum_{k=0}^{K-1} f_m(\hat z^n_k) - \int f_md\pi^n\right|<\epsilon\right\}.}
For every $m$, pick $K_m$ such that $\pr{\Omega_{m,K_m}}<\frac{\delta}{2^m}$; such a $K_m$ always exists due to Theorem \ref{thm:lln}. Define $\check\Omega = \Omega_0\bigcap \bigcap_{m=1}^\infty \Omega_{m,K_m}$. This immediately implies that 
\beqq{\pr{\check\Omega} = 1-\pr{\Omega_0^\complement\bigcup \bigcup_{m=1}^\infty \Omega_{m,K_m}^\complement}>1-\delta. }
For every $\omega\in \check\Omega$, pick $m$ sufficiently large such that $m\geq M_\epsilon(\omega)$, $m\geq \bar M$, and pick any $K = K_m$. We use \eqref{eqn:kmdiff} to conclude that 
\beqq{\Bigg|\frac{1}{K}\sum_{k=0}^{K-1}f(\hat z^n_k) - \int fd\pi^n\Bigg|<3\epsilon \text{ for all } \omega\in \check\Omega.}
Since $\pr{\omega\in \check\Omega} >1-\delta$, the proof of the lemma is complete.
\end{proof}
We can now complete the proof of Theorem \ref{thm:bounded} using the result above. Pick $f\in C(\ALP X)$ and define $f_+,f_-:\ALP X\rightarrow[0,\infty)$ such that $f_+ = \max\{f,0\}$ and $f_- = \max\{-f,0\}$. This immediately yields $f = f_+ - f_-$. Further, we have
\beqq{\frac{1}{K}\sum_{k=0}^{K-1} f(\hat z^n_k) & = \frac{1}{K}\sum_{k=0}^{K-1} f_+(\hat z^n_k) -\frac{1}{K}\sum_{k=0}^{K-1} f_-(\hat z^n_k), \\
\int fd\pi^n & = \int f_+d\pi^n - \int f_-d\pi^n.}
Now, it is easy to conclude from Lemma \ref{lem:positive} that the convergence in probability result in \eqref{eqn:asconv} holds for both $f_+$ and $f_-$, due to which the convergence in probability result holds for the function $f$ itself.

\section{Conclusion}\label{sec:conclusion}
In this paper, we studied the convergence of random sequences generated from certain RSAs used in machine and reinforcement learning problems. If the randomization device used within the algorithm is independent at every iteration, and the maps do not change (for instance, the stepsize is taken as constant), then the random sequence generated can be viewed using the lens of Markov chains. We leveraged the theory of Feller Markov chains to deduce many interesting characteristics of the random sequence and their distributions. Specifically, under reasonable conditions, we showed that the entire random sequence is close to the sequence generated by the exact algorithm with high probability for $n$ sufficiently large. We further showed that the average of the random sequence converges to the mean of the invariant distribution if the sequence if there exists a unique invariant measure of the Markov chain.

We expect that the results presented here can be applied to MDPs over continuous state and action spaces (referred to as continuous MDPs). Indeed, finite time guarantees of empirical value iteration for continuous MDPs with function approximator have been presented in \cite{munos2008finite,sharma2019empirical,sharma2019approximately,haskell2019universal,sharma2019approximate,haskell2017randomized,gupta2018probabilistic} under a variety of assumptions on the MDPs and performance criteria. Convergence of asynchronous algorithm for continuous MDPs with non-parametric function approximation is presented in \cite{shah2018q}. It will be interesting to investigate if the output of these algorithms satisfy the sufficient conditions for Theorem \ref{thm:main} and \ref{thm:lln}. It will also be interesting to apply the results presented here to variance reduced algorithms \cite{sidford2018near,wainwright2019variance,sato2019riemannian,wang2019stochastic,wang2018spiderboost} that have been developed recently.

Another problem left for future research is to determine bounds on $\pr{\rho(\hat z^n_k,y_k)\geq \epsilon}$ for any $\epsilon>0$ for every $n$ and $k$. Such bounds would unify our understanding of finite time guarantees for RSAs and allow us to improve the existing algorithms. We hope that the unified framework developed in this paper will be useful for analyzing many other learning algorithms in the future, particularly for analyzing MDPs over compact uncountable state and action spaces.

\appendix

\section{Proof of Lemma \ref{lem:weak}} \label{app:weak} 
Since $\mu_n\circ\Pi_l^{-1}\rightarrow \psi\circ\Pi_l^{-1}$ in weak topology as $n\rightarrow\infty$, we conclude that $\Big(\mu_n\circ\Pi_l^{-1}\Big)_{n\in\Na}$ is tight. For a fixed $\epsilon>0$, let $F_2\subset\ALP X$ be the compact set such that
\beqq{\mu_n\circ\Pi_l^{-1}(\ALP X^l\times F_\epsilon) > 1-\frac{\epsilon}{4\|g\|_\infty} \text{ for every } n\in\Na.}
We now need the following result.
\begin{lemma}\label{lem:gsuppr}
If Assumption \ref{assm:suppr} holds, then for any $g\in U_b(\ALP X)$, compact set $K\subset\ALP X$ and $\epsilon>0$, there exists $N\in\Na$ such that
\beqq{\Big|\ex{g(\hat T^n_k(x))} - g(T(x)) \Big|<\epsilon \text{ for all } x\in K.}
\end{lemma}
\begin{proof}
Since $g$ is uniformly continuous, for every $\epsilon>0$, there exists a $\delta_{\epsilon}>0$ such that for any $x,x'\in\ALP X$ with $\rho(x,x')<\delta_\epsilon$, we have $|g(x) - g(x')|<\epsilon$. Since Assumption \ref{assm:suppr} holds, there exists $ N_{\epsilon}(g,K)$ such that
\beqq{\sup_{x\in K}\pr{\rho(\hat T^n_k(x),T(x)) >\delta_\epsilon}<\frac{\epsilon}{2\|g\|_\infty} \text{ for all } n\geq N_{\epsilon}(g,K).}This implies
\beqq{\Big|\ex{g(\hat T^n_k(x)) - g(T(x))}\Big|  & \leq  \int |g(\hat T^n_k(x)) - g(T(x))| \pr{d\omega}\\
& \leq \epsilon \pr{\rho(\hat T^n_k(x),T(x))< \delta_\epsilon} \\
& \qquad + 2\|g\|_\infty \pr{\rho(\hat T^n_k(x),T(x))\geq \epsilon}\\  
& < 2\epsilon.}
The proof of the lemma is complete. 
\end{proof}

We are now in a position to prove the result. Consider the following expressions
\beqq{ & \int_{\ALP X^{l+1}}  \Big|\ex{g(\hat T^n_l(x_l))} - g(T(x_l))\Big| \mu_n\circ \Pi_l^{-1}(dx_{0:l})\\
& = \int_{\ALP X^l\times F_\epsilon}  \Big|\ex{g(\hat T^n_l(x_l))} - g(T(x_l))\Big| \mu_n\circ \Pi_l^{-1}(dx_{0:l}) \\
& + \int_{\ALP X^l\times F_\epsilon^\complement}  \Big|\ex{g(\hat T^n_l(x_l))} - g(T(x_l))\Big| \mu_n\circ \Pi_l^{-1}(dx_{0:l})\\
& \leq \frac{\epsilon}{2}+\frac{\epsilon}{2} = \epsilon.}
The proof of the lemma is complete.

\section{Proof of Lemma \ref{lem:gkN}}\label{app:gkN}
Define $g_p(\hat z^n_{-i}) := 0$ for all $i\in\Na$ and $p\in\{0,1,\ldots\}$. Let us define random variables $u_{p,l}$ for $l,p = 0,1,\ldots$ as
\beqq{u_{p,l} = g_p(\hat z^n_l) - g_{p+1}(\hat z^n_{l-1}).}
Thus, $u_{p,0} = g_p(\hat z^n_0)$. By definition of $u_{p,l}$, we have for any $l,k\geq 0$
\beq{g_0(\hat z^n_l) - g_k(\hat z^n_{l-k}) & = \Big(g_0(\hat z^n_l)-g_1(\hat z^n_{l-1})\Big)+ \ldots + \Big(g_{k-1}(\hat z^n_{l-k+1}) - g_k(\hat z^n_{l-k})\Big)\nonumber\\
& = u_{0,l} + \ldots + u_{k-1,l} = \sum_{p = 0}^{k-1}u_{p,l}.\label{eqn:llk}}
We use the above expression to establish the following identity.
\begin{lemma}
For $M\leq N$, we have
\beq{\label{eqn:MN}\frac{1}{N}\sum_{l=0}^{N-1}\Bigg( g_0(\hat z^n_l) - \bar g_M(\hat z^n_l)\Bigg) =  \frac{1}{M}\sum_{k=0}^{M-1}\sum_{p = 0}^{k-1} \left(\frac{1}{N}\sum_{l=0}^{N-1}u_{p,l}\right) -\frac{1}{M}\sum_{k=0}^{M-1}\Bigg(\frac{1}{N}\sum_{l=N-k}^{N-1}g_k(\hat z^n_l)\Bigg).}
\end{lemma}
\begin{proof}
For any $l\in\{0,1,\ldots\}$, we have
\beqq{ g_0(\hat z^n_l) - \bar g_M(\hat z^n_l) & =  g_0(\hat z^n_l) - \frac{1}{M} \sum_{k=0}^{M-1}g_k(\hat z^n_l) = \frac{1}{M}\sum_{k=0}^{M-1}\Big(g_0(\hat z^n_l) - g_k(\hat z^n_l)\Big)\\
& = \frac{1}{M}\sum_{k=0}^{M-1}\bigg( \Big(g_0(\hat z^n_l) )  - g_k(\hat z^n_{l-k})\Big) + \Big(g_k(\hat z^n_{l-k}) - g_k(\hat z^n_l)\Big)\bigg)\\
& = \frac{1}{M}\sum_{k=0}^{M-1}\bigg( \Big(g_0(\hat z^n_l) )  - g_k(\hat z^n_{l-k})\Big) + \Big(g_k(\hat z^n_{l-k}) - g_k(\hat z^n_l)\Big)\bigg).}
This yields
\beq{\label{eqn:B1}\frac{1}{N}\sum_{l=0}^{N-1}\Bigg( g_0(\hat z^n_l) - \bar g_M(\hat z^n_l)\Bigg) & = \frac{1}{NM}\sum_{k=0}^{M-1}\sum_{l=0}^{N-1}\bigg( \Big(g_0(\hat z^n_l) )  - g_k(\hat z^n_{l-k})\Big) + \Big(g_k(\hat z^n_{l-k}) - g_k(\hat z^n_l)\Big)\bigg).
}
For any $l,k\geq 0$, \eqref{eqn:llk} yields
\beqq{\frac{1}{N}\sum_{l=0}^{N-1}\Big(g_0(\hat z^n_l) - g_k(\hat z^n_{l-k})\Big) =  \frac{1}{N}\sum_{l=0}^{N-1}\sum_{p = 0}^{k-1}u_{p,l} = \sum_{p = 0}^{k-1} \left(\frac{1}{N}\sum_{l=0}^{N-1}u_{p,l}\right).}
Consider the second summand in \eqref{eqn:B1}. For any $k\in\Na$, we have
\beqq{\frac{1}{N}\sum_{l=0}^{N-1}\Big(g_k(\hat z^n_{l-k}) - g_k(\hat z^n_l)\Big) = -\frac{1}{N}\sum_{l=N-k}^{N-1}g_k(\hat z^n_l). }
These expressions immediately establish the equality in \eqref{eqn:MN}.

\end{proof}

\begin{lemma}\label{lem:ulp}
For a fixed $p\in\{0,1,\ldots\}$, we have
\beqq{\lim_{N\rightarrow\infty}\frac{1}{N}\sum_{l=0}^{N-1} u_{p,l} = 0 \quad \mathbb{P}\text{-almost surely}.}
\end{lemma}
\begin{proof}
For any $p\in\{0,1,\ldots\}$ and $l\in\Na$, let $\ALP F_{p,l-1}$ denote the $\sigma$-algebra generated by $(\hat z^n_0,\ldots,\hat z^n_{l-1},u_{p,0},\ldots,u_{p,l-1})$. We show that $\{u_{p,l}\}_{l=1}^\infty$ forms a martingale with respect to the $\sigma$-algebra $\ALP F_{p,l-1}$.
By the definition of $u_{p,l}$, we immediately have for any $l\geq 1$,
\beq{\label{eqn:upl}\ex{u_{p,l}|\hat z^n_{l-1}} =\ex{g_p(\hat z^n_l) - g_{p+1}(\hat z^n_{l-1})|\hat z^n_{l-1}} = 0.}
This implies
\beqq{\ex{u_{p,l}|\ALP F_{p,l-1}} = 0, \quad \ex{(u_{p,l})^2} \leq 2\|f\|_\infty^2 \text{ for $p\in\{0,1,\ldots\}$ and $l\in\Na$}.}
The proof then is an immediate consequence of the strong law of large numbers for martingales in \cite[p. 66]{loeve1977ii}.
\end{proof}

Now note that for any $M\in\Na$, as $N\rightarrow\infty$, the right side of \eqref{eqn:MN} converges to 0 almost surely by Lemma \ref{lem:ulp}. This yields the result.

\bibliographystyle{siamplain}
\bibliography{b/2019rl,b/constantrl,b/constantsgd,b/edp,b/math,b/myconf,b/myjournal,b/optbook,b/prob,b/probbook,b/ql,b/gamefp,b/sgd}
\end{document}